%% file: main.tex
\documentclass[twoside,11pt]{article}

\usepackage{jmlr2e}

\usepackage{amsmath, amssymb}
\usepackage{enumitem}
\usepackage{algorithm}
\usepackage[]{algpseudocode}
\usepackage{xcolor}
\usepackage{wrapfig}
\usepackage{booktabs}
\usepackage{subfigure}
\usepackage{cancel}
\usepackage{mathtools}
\usepackage{doi}

\allowdisplaybreaks

\usepackage{nicematrix}

\newtheorem{assumption}{Assumption}

\algnewcommand{\parState}[1]{\State%
    \parbox[t]{\dimexpr\linewidth-\algmargin}{\strut #1\strut}}

\newcommand{\twocolour}{red}
\newcommand{\threecolour}{blue}
\DeclareMathOperator*{\p}{\mathbf{P}}
\DeclareMathOperator*{\E}{\mathbf{E}}
\DeclareMathOperator*{\e}{\mathrm{e}}
\DeclareMathOperator*{\argmin}{arg\,min}
\DeclareMathOperator*{\argmax}{arg\,max}
\renewcommand{\d}[1]{\ensuremath{\operatorname{d}{#1}}}

\usepackage{tikz}  
\usepackage{pgfplots}
\usepackage{pgfplotstable}
\usepackage{tikz-3dplot}
\pgfplotsset{compat=1.13}
\usetikzlibrary{patterns, decorations.pathreplacing, decorations.fractals, arrows, automata, shadings, matrix, chains, scopes}
\usepgfplotslibrary{ternary}

\ShortHeadings{ChronosPerseus}{Kohar, Rivest, and Gosselin}
\firstpageno{1}

\begin{document}

\title{ChronosPerseus: Randomized Point-based Value Iteration with Importance Sampling for POSMDPs}

\author{\name Richard Kohar \email richard@math.kohar.ca \\
\AND 
\name Fran\c{c}ois Rivest \email francois.rivest@\{mail.mcgill.ca, rmc.ca\} \\
\AND
\name Alain Gosselin \email alain.gosselin@rmc.ca \\
\addr Department of Mathematics and Computer Science\\
Royal Military College of Canada\\
P.O. Box 17000 Station Forces\\
Kingston, Ontario, Canada\\
K7K 7B4}

\editor{}

\maketitle

\begin{abstract}
In reinforcement learning, agents have successfully used environments modeled with Markov decision processes (MDPs). However, in many problem domains, an agent may suffer from noisy observations or random times until its subsequent decision.  While partially observable Markov decision processes (POMDPs) have dealt with noisy observations, they have yet to deal with the unknown time aspect. Of course, one could discretize the time, but this leads to Bellman's Curse of Dimensionality. To incorporate continuous sojourn-time distributions in the agent's decision making, we propose that partially observable semi-Markov decision processes (POSMDPs) can be helpful in this regard. We extend \citet{Spaan2005a} randomized point-based value iteration (PBVI) \textsc{Perseus} algorithm used for POMDP to POSMDP by incorporating continuous sojourn time distributions and using importance sampling to reduce the solver complexity. We call this new PBVI algorithm with importance sampling for POSMDPs---\textsc{ChronosPerseus}. This further allows for compressed complex POMDPs requiring temporal state information by moving this information into state sojourn time of a POMSDP. The second insight is that keeping a set of sampled times and weighting it by its likelihood can be used in a single backup; this helps further reduce the algorithm complexity. The solver also works on episodic and non-episodic problems. We conclude our paper with two examples, an episodic bus problem and a non-episodic maintenance problem.
\end{abstract}


\section{Introduction}

Time is critical to cognitive agents making decisions \citep{ManiadakisTrahanias2011}. Planning not only requires \textit{what} to do but also \textit{when} to do it \citep{RivestKohar2020}. For instance, when should we give up waiting for a bus and start walking? When should we perform some maintenance on a particular piece of machinery? When should an autonomous robotic vacuum cleaner start to return to its base before it runs out of battery life? In each of these mentioned situations, the agent requires a time model to aid in the decision process.

In artificial intelligence, a common model for planning in a stochastic environment is the Markov decision process (MDP) \citep{RussellNorvig2010, Sutton2018}. It is a formal framework of sequential decision making that not only considers which sequence of actions, or a policy, gains the most immediate reward but also future rewards \citep{Bellman1957a}.  Time is considered by discounting the reward, but to consider the issue of \textit{when} to act, we could augment the state space with time steps leading to what Bellman calls the Curse of Dimensionality.  

Two well-known extensions to MDPs are semi-Markov MDP (SMDP) \citep{Howard1963} and partially observable MDP (POMDP) \citep{Aoki1965, Astrom1965}. The former allows for transitions to have sojourn time distributions, while the latter allows for partially observable states (incomplete information). In reinforcement learning, SMDPs are primarily used with options \citep{Sutton1999b, Precup2000a}, where the underlying MDP of the micro-actions defines the transitions of macro-actions. The SMDP, in that case, represents a temporal compression of the underlying MDP. In contrast, POMDPs are computationally expensive \citep{Papadimitriou1987, Madani1999}, as the solvers usually construct a policy over a belief state space that must cover all the combinations of states. 

POMDPs are not widely used in practice due to their computational complexity \citep{Papadimitriou1987, Madani1999, Sutton2018}. However, with the development of approximate solutions methods such as point-based value iteration by \citet{Pineau2003} leading to \textsc{Perseus} by \citet{Spaan2005a}.

Cognitive agents, as well as many scheduling and maintenance problems, could naturally benefit from the ability to solve problems that mix both partial observability and rich temporal transitions. The partially observable semi-Markov decision process (POSMDP) extends POMDP by including stochastic sojourn time between transitions. As we will show, this sojourn time can further be used in the belief state update. \citet{White1976} first proposed the finite-horizon discrete-time POSMDP model, and his algorithm for computing the optimal cost and an optimal policy were based on procedures for POMDPs by \citet{Sondik1971}. Half a decade later, \citeauthor{Wakuta1981} defined the infinite-horizon continuous-time POSMDP model with the average-cost criterion \citep{Wakuta1981} and the discounted-cost criterion \citep{Wakuta1982}. The definition and notation of POSMDP that we present in Section~\ref{sec:POSMDPDefinition} is influenced by \citet{Wakuta1982}, \citet{Hernandez-Lerma1989Book}, and \citet{Yu2006}.

While \citet{Zhang2017} developed a POSMDP solver based on \textsc{Perseus}, we introduce importance sampling to use the collected samples more efficiently. This leads us to the significant contribution of this paper, which is our importance sampling point-based POSMDP solver called \textsc{ChronosPerseus} that we will present in Section~\ref{sec:ChronosPerseus}. This allows us to include temporal properties to POMDPs without the added complexity that would be generated by extending the temporal information through a large expansion of the state space dimensions. 

The remainder of this paper is organized as follows. First, in Section~\ref{sec:POSMDPDefinition} we formally introduce the framework, the notation, and the dynamics for the POSMDP model. Then, we discuss how we handle concepts of time in Section~\ref{sec:POSMDP:time} and partial observability in Section~\ref{sec:POSMDP:PartialObservability}. We discuss the difference between the process's history and the observable history in Section~\ref{sec:POSMDP:history}. In Section~\ref{sec:POSMDP:beliefstate}, we give belief update rules that do incorporate and do not incorporate sojourn time.  We believe that the sojourn time is important information that can be incorporated into the belief update, which in turn affects the agent's optimal decision. Section~\ref{sec:POSMDP:reward} discusses the reward and the assumptions required to ensure the reward function is well-defined. For the value function of a POSMDP in Section~\ref{sec:POSMDP:ValueFunction}, we show that depending on the approach taken, either extending the POMDP model with the addition of continuous time.  Finally, we describe in Section~\ref{sec:ChronosPerseus} the main contribution of this paper---the \textsc{ChronosPerseus} algorithm. In Section~\ref{sec:Applications}, we conclude this paper with some solved POSMDP examples: the bus problem of an agent deciding to stay on the bus or ride its bike, and maintenance of water filters in the real world based on parameters from \citet{Zhang2017}. The first one includes a mixture of random continuous and fixed sojourn times as well as mixed-observability on an episodic task, while the second includes continuous sojourn times and observation space on a non-episodic problem.

The code to the \textsc{ChronosPerseus} algorithm and the examples are available online (\url{https://github.com/rkohar/ChronosPerseus}).

\section{Partially Observable Semi-Markov Decision Processes}\label{sec:POSMDPDefinition}

A partially observable semi-Markov decision process (POSMDP) is an 11-tuple $$\left\langle \mathcal{S}, \mathcal{A}, \mathcal{K}, \mathcal{O}, Q, G, G_0, R, \xi_0, \beta, N \right\rangle$$ where
\begin{itemize}[leftmargin=2.9cm]
	  \item[$\mathcal{S}$] is the Borel \textit{state space}, and the elements of $\mathcal{S}$ are called \textit{states}
	  \item[$\mathcal{A}$] is the Borel \textit{action space}, and the elements of $\mathcal{A}$ are called \textit{actions}. To each $s \in \mathcal{S}$, we associate a nonempty Borel-measurable subset $\mathcal{A}(s) \subseteq \mathcal{A}$, whose elements are the \textit{admissible actions} for the agent in state $s$.
		\item[$\mathcal{K}$] is the set of admissible state-action pairs, and it is assumed to be a Borel subset $\mathcal{K} \subseteq \mathcal{S} \times \mathcal{A}$. 
    \item[$\mathcal{O}$] is the Borel \textit{observation space}, and the elements of $\mathcal{O}$ are called observations
    \item[$Q(\d \tau, \d s' \mid s, a)$] denotes the sojourn time-state (Borel-measurable) stochastic kernel on $\mathbb{R}_{> 0} \times \mathcal{S} $ given $\mathcal{S} \times \mathcal{A}$
	  \item[$G(\d o \mid a, s')$] is the observation (Borel-measurable) stochastic kernel on $\mathcal{O}$ given $\mathcal{A} \times \mathcal{S}$
		\item[$G_0(\d o \mid s')$] denotes the initial observation (Borel-measurable) stochastic kernel on $\mathcal{O}$ given $\mathcal{S}$
  \item[$R(s,a)$] is the per-stage (bounded Borel-measurable) reward function given $\mathcal{S} \times \mathcal{A}$ ($R(s,a)$ could be further broken down into a leap sum and a continuous reward, see Section~\ref{sec:POSMDP:reward}.)
	\item[$\xi_0$]  is the (\textit{a priori}) initial (belief) state distribution ($\xi_0 \in \p(\mathcal{S})$)
	\item[$\beta$] is the discounting rate where $\beta \in [0,1)$.
	\item[$N$] is the planning horizon. It could be finite, or $N = \infty$.
\end{itemize}

Given the model $\left\langle \mathcal{S}, \mathcal{A}, \mathcal{K}, \mathcal{O}, Q, G, G_0, R, \xi_0, \beta, N \right\rangle$, the dynamics of the POSMDP proceed according to Algorithm~\ref{alg:DynamicsPOSMDP}. This involves at each decision epoch $n$ choosing an action $a_n$, accruing reward $R(s_n, a_n)$ as the state changes from $s_n$ to $s_{n+1}$ in $\tau_n$ amount of time, and partially observing $s_{n+1}$ as $o_{n+1}$. The agent uses all the information available up to decision epoch $n$, namely the observable history $h_n$, to choose action $a_n = \pi_n(h_n)$ using policy $\pi_n$. With the dynamics of the POSMDP specified by Algorithm~\ref{alg:DynamicsPOSMDP} and illustrated in Fig.~\ref{fig:POSMDP:GraphicalModel}, we denote the sequence of policies that the agent uses from decision epoch $0$ to $n-1$ as $\pi = (\pi_0, \pi_1, \ldots, \pi_{n-1})$. The set of all admissible policies is denoted $\Pi$. 

\begin{algorithm}
\caption{Dynamics of POSMDP}\label{alg:DynamicsPOSMDP}
In the beginning $n = 0$, the state $s_0$ is simulated from an initial (belief) state distribution $\xi_0$.

For each decision epoch $n = 1, 2, \ldots, N$:
\begin{enumerate}
 \item Based on the observable history
   \begin{align*}
	   h_0 &= (\xi_0)\\
		 h_n &= (\xi_0, a_1, t_1, o_1, \ldots, a_n, t_n, o_n),
	 \end{align*}
	 the agent performs action
	  $$a_n = \pi_n(h_n) \in \mathcal{A} \qquad n = 1, 2, \ldots, N.$$
	 Here $\pi_n$ denotes a policy that the agent uses at decision epoch $n$.
 \item The agent obtains a reward $R(s_n, a_n)$ for choosing action $a_n$ at decision epoch $n$.
 \item The state evolves randomly with sojourn time-state transition probability
   \begin{equation}
    Q(\tau, s' \mid s, a) = \p(T_{n+1} - T_n \leq \tau, S_{n+1} = s' \mid S_n = s, A_n = a)
    \label{eq:POSMDP:definitionofQ}
    \end{equation}
	to the next state $s_{n+1}$ that decision epoch $n+1$.
 \item The agent records a noisy observation $O_n \in \mathcal{O}$ of the state $S_{n+1}$ according to
   $$G(o \mid a, s') = \p(O_n = o \mid A_n = a, S_{n+1} = s').$$
 \item The agent updates its history as
   $$h_{n+1} = (h_n, a_{n+1}, t_{n+1}, s_{n+1}).$$
	If $n < N$, then set $n$ to $n+1$, and go back to step (a).
	
	If $n = N$, then the agent receives the last reward and the process terminates.
\end{enumerate}
\end{algorithm}

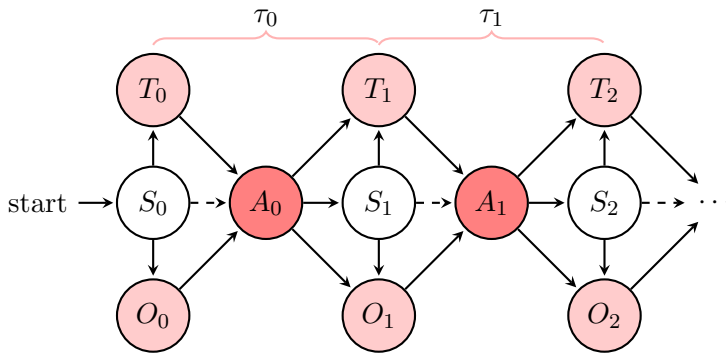
\begin{figure}[h!t]
\centering
\input{Fig/POSMDPdiagram2.tex}
\caption{Graphical model of POSMDP. The states $S_i$ are hidden from the agent, while the observations $O_i$ and times $T_i$ are observable. The sojourn times are $\tau_i$ and the agent's actions are $A_i$. The solid arrows $\rightarrow$ represent the direct influence from one element to another, while the dashed arrows $\dashrightarrow$ represent indirect influence. For example, action $A_0$ is directly influenced by time $T_0$ and observation $O_0$ because this is the information that is available to the agent, while state $S_0$ is indirectly influencing action $A_0$ since the agent has to infer in which state it is in currently from time $T_0$ and observation $O_0$.}
\label{fig:POSMDP:GraphicalModel}
\end{figure}

\subsection{Time}\label{sec:POSMDP:time}
Unlike POMDPs where the sojourn times are constant, the sojourn times are random in a POSMDP; see Fig.~\ref{fig:POSMDP:POMDPvsPOSMDPtime}.
\begin{figure}[h!t]
\centering
\input{Fig/POMDPvsPOSMDPtime.tex}
\caption{The difference between POMDPs and POSMDPs for sojourn times. The sojourn time for a POMDP is constant whereas the sojourn time for a POSMDP is a random amount of time that follows some probability distribution.}
\label{fig:POSMDP:POMDPvsPOSMDPtime}
\end{figure}
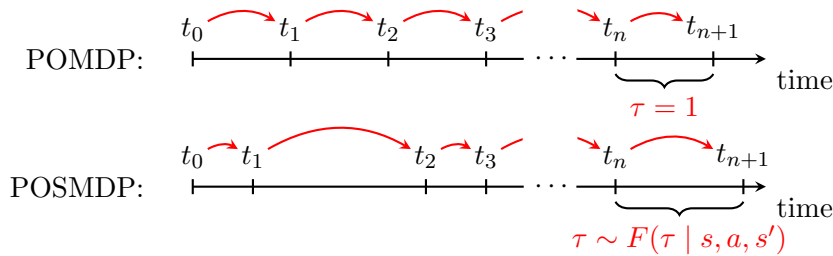

The time of the $n^\text{th}$ decision epoch is denoted by the random variable
\begin{equation}
T_n = \begin{cases}
               0                    & n = 0\\
               T_{n-1} + \tau_{n-1} & n = 1, 2, 3, \dots,\\
      \end{cases}
\end{equation}
in which the observed time of $T_n$ is denoted by $t_n$. The sojourn time from state $s$ to state $s'$ is a nonnegative real-valued random variable $\tau_n = T_{n+1} - T_n$ distributed by $F(\tau \mid s, a, s')$ conditional on $(s,a,s')$ such that 
\begin{equation} \label{eq:ConditionalSojournTime}
    Q(\tau, s' \mid s, a) = P(s' \mid s, a) F(\tau \mid s, a, s').
\end{equation}
From \eqref{eq:ConditionalSojournTime}, we assume implicitly that a sojourn time $\tau_n$ does not affect the next sojourn time $\tau_{n+1}$.

\subsection{Partial observability}\label{sec:POSMDP:PartialObservability}
For partial observability, we follow the same structure from POMDPs. We do not change the notion of partial observability with the introduction of random sojourn time between decision epochs. In our definition of the POSMDP in Section~\ref{sec:POSMDPDefinition}, the observation is independent of the previous state. If the observation depends on more information, such as $(s,a,\tau,s')$, then we could define $G(o \mid s, a, \tau, s')$ as
\begin{equation}
G(o \mid a, s') = \sum_{s \in \mathcal{S}} Q(\tau, s' \mid s, a) G(o \mid s, a, \tau, s').
\label{eq:POSMDP:Q-PF}
\end{equation}

\subsection{History}\label{sec:POSMDP:history}
A realization of this process is
$$(s_0, o_0, a_0, \tau_0, s_1, o_1, a_1, \tau_1, s_2, o_2, a_2, \tau_2, \ldots) \in (\mathcal{S} \times \mathcal{O} \times \mathcal{A} \times \mathbb{R}_{>0})^\infty \triangleq \Omega.$$
However, we are not able to rely on the unobservable states $s_0$, $s_1$, $s_2$, $\ldots,$ so we introduce the concept of observable histories. We define the space of \textit{observable histories} until the $n^\text{th}$ decision epoch by 
\begin{equation}
\mathcal{H}_n = \begin{cases}
       \p(\mathcal{S}) \times \mathcal{O}, & n = 0;\\
       \mathcal{H}_{n-1} \times \mathcal{A} \times \mathbb{R}_{>0} \times \mathcal{O}, & n = 1, 2, 3, \ldots.
       \end{cases}
\end{equation}
Thus, the element $h_0 = (\xi_0, o_0) \in \mathcal{H}_0$ is an initial observed history, and an element $h_n \in \mathcal{H}_n$ is called an \textit{observed history} up to $n$, and it is denoted by
\begin{align}
h_n &= (\underbrace{\xi_0, o_0, a_0, \tau_0, o_1, a_1, \tau_1, \ldots, o_{n-1}, \tau_{n-1}}_{h_{n-1}}, a_{n-1}, \tau_n, o_n)\\
    &= (h_{n-1}, a_{n-1}, \tau_n, o_n).\label{eq:history}
\end{align}

\subsection{Belief State Representation}\label{sec:POSMDP:beliefstate}
The agent can only partially observe the state $s \in \mathcal{S}$ by inferring from its past and current observations.  Thus, a memory of past observations is required for an agent to optimally choose its actions in a partially observable environment.  

A na\"{i}ve approach would allow the agent to remember its sequence of sojourn times, observations, and actions.  However, this sequence can grow unbounded over time, which is not practical with finite memory.  Instead, we can summarize this information with sufficient statistics \citep{Aoki1965, Astrom1965}. A statistic is considered \textit{sufficient} when no other statistic calculated from the same sample provides any additional information to the estimated parameter value \citep{Fisher1922}.  With these sufficient statistics that summarize the observable history, we can construct a probability distribution of where the agent is in the state space.

\begin{definition}[Belief state] 
A belief state $\xi$ is a probability distribution over the state space $\mathcal{S}$. 
\end{definition}

If the state space $\mathcal{S} = \{s_1, s_2, \ldots, s_{|\mathcal{S}|}\}$ is a finite set, then the belief state $\xi$ is defined by
   \begin{equation*}
     \xi = \begin{bmatrix} \xi(s_1) \\ \xi(s_2) \\ \vdots \\ \xi(s_{|\mathcal{S}|})\end{bmatrix} \triangleq \begin{bmatrix} \p(S = s_1) \\ \p(S = s_2) \\ \vdots \\ \p(S = s_{|\mathcal{S}|})\end{bmatrix}.
   \end{equation*}
To ensure that $\xi$ is a probability distribution,
\begin{equation*}
\xi(s_i) \geq 0 \qquad \forall s_i \in \mathcal{S}, \qquad\qquad \text{and} \qquad\qquad \sum_{i = 1}^{|\mathcal{S}|} \xi(s_i) = 1.
\end{equation*}

This probability distribution encodes the agent's subjective probability of its location in the environment's state space from its history. 
\begin{equation}
\xi_n(s_i) = \p(S_n = s_i \mid \xi_0, o_0, a_0, o_1, a_1, \ldots, o_{n-1}, a_{n-1}, o_n).
\end{equation}
The belief state is a sufficient statistic of the observable history. In other words, the belief state carries all the information the agent requires to make its decision from the present state; hence, the process retains the Markov property.
\begin{proposition}
A belief state $\xi$ is a sufficient statistic of the agent's history, and given an action $a$, sojourn time $\tau$, and observation $o$, the next belief state is
\begin{equation*}
\xi_n(s_i) = \p(S_n = s_i \mid \xi_{n-1}, a_{n-1}, \tau_{n}, o_n)
\end{equation*}
\end{proposition}

We can update the belief state in two ways: with or without observing sojourn time $\tau$.

\subsubsection{The Belief Update without Sojourn Time}

The belief update for POMDP is usually given by
\begin{equation}
\xi(s' \mid a, o) = \frac{G(o \mid a, s') \displaystyle\sum_{s \in \mathcal{S}} \xi(s) P(s' \mid s, a) }{P(o \mid \xi, a)} \qquad \forall s' \in \mathcal{S}, \label{eq:POMDP:discreteStatediscreteAction:beliefState}
\end{equation}
where the denominator
\begin{equation}
P(o \mid \xi, a) = \sum_{s' \in \mathcal{S}} G(o \mid a, s') \displaystyle\sum_{s \in \mathcal{S}} \xi(s) P(s' \mid s, a) \label{eq:POMDP:discreteStatediscreteAction:beliefStateDenominator}
\end{equation}
is a normalization factor \citep{Spaan2012}. 

By replacing the state transition probability distribution $P$ with the sojourn-time state transition probability distribution $Q$, the belief update for POSMDP, if the transition time $\tau$ is not observed, would be given by
\begin{align}
\xi(s' \mid a, o) &= \frac{G(o \mid a, s')  \displaystyle\sum_{s \in \mathcal{S}} \xi(s) \int_0^\infty Q(\d \tau, s' \mid s, a)}{P(o \mid \xi, a)}\\
                  &= \frac{G(o \mid a, s') \displaystyle\sum_{s \in \mathcal{S}} \xi(s) P(s' \mid s, a) \int_0^\infty  F(\mathrm{d} \tau \mid s, a, s')}{P(o \mid \xi, a)}\\
                   &= \frac{G(o \mid a, s') \displaystyle\sum_{s \in \mathcal{S}} \xi(s) P(s' \mid s, a) }{P(o \mid \xi, a)}
\end{align}
since the integral of $F$ is one. Therefore, the belief update (including its denominator, the normalization factor) would remain unchanged in this case. However, it would make sense for the agent to use this information since different transitions may have distinguishable sojourn time distributions, indicative of the current state.

\subsubsection{The Belief Update with Sojourn Time}

\begin{theorem}
Given an action $a$ the agent performed, a sojourn time $\tau$, and an observation $o$ seen, the belief update for a particular state $s'$ is
\begin{equation*}
    \xi(s' \mid a, \tau, o) = \frac{G(o \mid a, s') \displaystyle\sum_{s \in \mathcal{S}} Q(\tau, s' \mid s, a) \xi(s)}{\displaystyle\sum_{s''\in \mathcal{S}}G(o \mid a, s'') \sum_{s \in \mathcal{S}} Q(\tau, s'' \mid s, a) \xi(s)}.
\end{equation*}
\end{theorem}

\begin{proof}
By definition, 
\begin{equation*}
    \xi(s' \mid a, \tau, o) \triangleq \p(s' \mid a, \tau, o).
\end{equation*}
Applying the conditional rule 
\begin{equation*}
\p(A \mid B) = \frac{\p(A, B)}{\p(B)},
\end{equation*}
we have
\begin{equation}
    \xi(s' \mid a, \tau, o) = \p(s' \mid a, \tau, o) = \frac{\p(a, \tau, o, s')}{\p(a, \tau, o)} = \frac{\p(a, \tau, o, s')}{\displaystyle\sum_{s'' \in \mathcal{S}} \p(a, \tau, o, s'')}.
    \label{eq:POSMDP:beliefupdate:proof}
\end{equation}

Using the multiplication rule $\p(A, B) = \p(A \mid B) \p (B)$, we have
\begin{equation*}
\p(s, a, \tau, o, s') = \p(o \mid s, a, \tau, s') \p(s,a,\tau, s')
\end{equation*}
and applying the multiplication rule again to the second term,
\begin{equation*}
 \p(s, a, \tau, o, s') = \underbrace{\p(o \mid s, a, \tau, s')}_{G(o \mid a, s')} \underbrace{\p(\tau, s' \mid s, a)}_{Q(\tau, s' \mid s, a)} \underbrace{\p(s, a)}_{\xi(s)}.
\end{equation*}
Since the observation $o$ only depends on action $a$ and the next state $s'$,
\begin{equation*}
  \p(o \mid s, a, \tau, s') = G(o \mid a, s'),
\end{equation*}
and by definition $\p(\tau, s' \mid s, a) = Q(\tau, s' \mid s, a)$ (see Eq.~\eqref{eq:POSMDP:definitionofQ}) and $P(s, a) = P(a \mid s) P(s) = \xi(s)$ since $P(a \mid s) = 1$ for the selected action $a$ (the policy is deterministic) and $P(s) = \xi(s)$ (our Bayesian estimate). Thus, 
\begin{equation*}
\p(s, a, \tau, o, s') = G(o \mid a, s') Q(\tau, s' \mid s, a) \xi(s).
\end{equation*}
and
\begin{equation}
\p(a, \tau, o, s') = \sum_{s \in \mathcal{S}} \p(s, a, \tau, o, s') = G(o \mid a, s') \sum_{s \in \mathcal{S}} Q(\tau, s' \mid s, a) \xi(s).
\label{eq:POSMDP:beliefupdate:proof:P(xi,a,tau,o,s')}
\end{equation}
Substituting Eq.~\eqref{eq:POSMDP:beliefupdate:proof:P(xi,a,tau,o,s')} into Eq.~\eqref{eq:POSMDP:beliefupdate:proof} yields the required result.
\end{proof}

We can further decomposed the belief update function when $Q(\tau, s' \mid s, a)$ is given as $P(s' \mid s, a)$ and $F(\tau \mid s, a, s')$ by

\begin{align}
\xi(s' \mid a, \tau, o) &= \frac{G(o \mid a, s') \displaystyle\sum_{s \in \mathcal{S}} Q(\tau, s' \mid s, a) \, \xi(s)}{P(o \mid \xi, a, \tau)} \qquad \forall s' \in \mathcal{S}\label{eq:POSMDP:discreteStatediscreteAction:beliefState:POMDP2SMDP:Q}\\
 &= \frac{G(o \mid a, s') \displaystyle\sum_{s \in \mathcal{S}} P(s' \mid s, a) F(\tau \mid s, a, s') \, \xi(s)}{P(o \mid \xi, a, \tau)} \qquad \forall s' \in \mathcal{S},\label{eq:POSMDP:discreteStatediscreteAction:beliefState:POMDP2SMDP}
\end{align}
where the denominator
\begin{align}
P(o \mid \xi, a, \tau) &= \sum_{s' \in \mathcal{S}} G(o \mid a, s') \displaystyle\sum_{s \in \mathcal{S}} Q(\tau, s' \mid s, a)\, \xi(s)\\
 &= \sum_{s' \in \mathcal{S}} G(o \mid a, s') \sum_{s \in \mathcal{S}} P(s' \mid s, a) F(\tau \mid s, a, s')\, \xi(s) \label{eq:POSMDP:discreteStatediscreteAction:beliefState:POMDP2SMDP:denominator}
\end{align}
is a normalization factor similar to \citet{Wakuta1982}. 

Note that this is similar to the belief update for POMDP; for instance, Eq.~\eqref{eq:POSMDP:discreteStatediscreteAction:beliefState:POMDP2SMDP} is the same as Eq.~\eqref{eq:POMDP:discreteStatediscreteAction:beliefState}, with the addition of the sojourn-time cumulative probability distribution $F(t \mid s, a, s')$ in the numerator and the denominator. This is important as Eq.~\eqref{eq:POSMDP:discreteStatediscreteAction:beliefState:POMDP2SMDP} allows the agent to use the observed transition time to better estimate its real underlying state. Our solver is using this more complete approach.

\subsection{Reward}\label{sec:POSMDP:reward}

\subsubsection{Discount reward criterion}
Let $R(s,a)$ denote the expected total discounted reward between two decision epochs, given that the agent is in state $s$ and it chooses to perform action $a$, which can be expressed by
\begin{equation}
 R(s,a) = \underbrace{r_1(s,a)}_{\mathclap{\substack{\text{\textcolor{\twocolour}{(a)}} \\ \text{\textcolor{\twocolour}{immediate}} \\ \text{\textcolor{\twocolour}{reward}}}}} + \underbrace{\int\limits_{\mathclap{s' \in \mathcal{S}}} \int_0^\infty \left(\int_0^\tau {\e}^{-\beta t} r_2(t \mid s,a,s') \d t \right) Q(\d \tau, \d s' \mid s, a).}_{\mathclap{\substack{\text{\textcolor{\twocolour}{(b)}} \\ \text{\textcolor{\twocolour}{expected discounted reward reward from $s$ to $s'$ over $\tau$ amount of time}}}}} 
 \label{eq:SMDP:RewardFunction:general}
\end{equation} Note that the discount rate $\beta$ is given as part of the model. The reward function can be thought of two parts:
\begin{enumerate}
  \item the lump sum reward $r_1(s,a)$ that is received immediately upon the agent performing action $a$ in state $s$; and
	\item the continuous reward rate $r_2(\tau \mid s,a,s')$ that is received by the agent continuously over the sojourn time $\tau$ from state $s$ to $s'$ under action $a$.
\end{enumerate}
Since the agent cannot observe the state $s \in \mathcal{S}$ directly, we can rewrite Eq.~\eqref{eq:SMDP:RewardFunction:general} for a given belief $\xi \in \triangle$ and action $a \in \mathcal{A}$ as
\begin{equation}
R(\xi, a) = \int\limits_{\mathclap{s \in \mathcal{S}}} \xi(\d s) R(s,a),
\end{equation}
or if it is a discrete set of states, it can be written as
\begin{equation}
R(\xi, a) = \sum_{s \in \mathcal{S}} \xi(s) R(s,a).
\end{equation}

We need a few assumptions to ensure that the expected rewards are well-defined. Time is the safeguard that prevents everything from happening at once. So, we borrow the following assumption for SMDPs by \citet{Ross1970} that allows only a finite number of decision epochs during a finite sojourn time.
\begin{assumption}\label{assumption:POSMDP:finiteNumberDecisionEpoch}
There exists a sojourn time $\tau > 0$ and $\epsilon > 0$ such that
\begin{equation*}
\int\limits_{\mathclap{s' \in \mathcal{S}}} Q(\tau, \d s' \mid s, a) \leq 1 - \epsilon \qquad \forall (s,a) \in \mathcal{K}.
\end{equation*}
\end{assumption}
In other words, Assumption~\ref{assumption:POSMDP:finiteNumberDecisionEpoch} says that for every (admissible) state-action pair $(s,a) \in \mathcal{K}$, there is a positive probability of at least $\epsilon$ that the transition time to state $s'$ will be greater than $\tau$.

We also impose the condition that the per-stage reward $R(s,a)$ is bounded for every state-action pair $(s, a) \in \mathcal{K}$.
\begin{assumption}\label{POSMDP:assumption:rewardBoundM}
There exists a constant $M$ such that for $\beta = 0$, 
\begin{equation*}
 \sup_{(s, a) \in \mathcal{K}} |R(s,a)| < M.
\end{equation*}
\end{assumption}
This assumption is satisfied if $r_1$ and $r_2$ in Eq.~\eqref{eq:SMDP:RewardFunction:general} are bounded, and $$\E(\tau_n \mid S_n = s, A_n = a) < \infty \qquad \forall (s, a) \in \mathcal{K}.$$   In other words, we would like each expected sojourn time $\E(\tau_n)$ to be bounded. Therefore, we make the following assumption.
\begin{assumption}
\begin{equation*}
\sup_{(s, a) \in \mathcal{K}} \E(\tau_n \mid S_n = s, A_n = a ) < \infty.
\end{equation*}
\end{assumption}

\section{ChronosPerseus}\label{sec:ChronosPerseus}
In this paper, we present \textsc{ChronosPerseus}, an approximate point-based value iteration algorithm with importance sampling developed to solve POSMDPs.  It follows other point-based value iteration methods in that it implements a belief state collection by simulating the environment \citep{Pineau2003, Pineau2006, Spaan2005a, Shani2013}, but additionally, it also collects sample sojourn times that are part of the belief update. Once the collection is complete, an approximate value function is constructed for which a policy is found. First, we will describe below the exact value function of a POSMDP followed by its construction with linear vectors (known as $\alpha$-vectors in the POMDP literature). Then, we will construct an approximate value function by sampling beliefs and sojourn times, which will lead to our summary of the \textsc{ChronosPerseus} algorithm.

\subsection{The Value Function of POSMDP}\label{sec:POSMDP:ValueFunction}

The optimal value function---otherwise known as the \citet{Bellman1957} equation---for a finite state, action, and observation POMDP is 
\begin{equation}
V^*(\xi) = \max_{a \in \mathcal{A}(s)}\left[\sum_{s \in \mathcal{S}} \xi(s) R(s, a) + \sum_{o \in \mathcal{O}} \sum_{s' \in \mathcal{S}}  G(o \mid a, s')
         \sum_{s \in \mathcal{S}} \xi(s) P(s' \mid s, a) \gamma V^\pi\big(\xi(\cdot \mid a,o)\big)\right]
\label{eq:POMDP:discreteStatediscreteAction:valueFunctionOptimal}
\end{equation}
for all belief states $\xi$ in the belief simplex $\triangle$.
When Eq.~\eqref{eq:POMDP:discreteStatediscreteAction:valueFunctionOptimal} holds for every belief state in the belief simplex, we are ensured the solution is optimal. If we were to compute the value function over the continuous belief space, belief by belief, it may seem intractable. However, \citet{Sondik1971} shows that in this case, the POMDP optimal value function is convex, and that it can be parameterized by a finite set of hyperplanes.

In order to adapt Eq.~\eqref{eq:POMDP:discreteStatediscreteAction:valueFunctionOptimal} for a POSMDP, the discount factor must take the sojourn time probability $F(\d \tau \mid s, a, s')$ into account. It thus becomes  
\begin{equation}
\gamma(s, a, s') = \int_0^\infty {\e}^{-\beta \tau} F(\d \tau \mid s, a, s').
\end{equation}
Then, the Bellman equation for POSMDP is
\begin{equation}{\scriptsize
V^*(\xi) = \max_{a \in \mathcal{A}(\xi)}\Big[\underbrace{\sum_{s \in \mathcal{S}} \xi(s) R(s, a)}_\text{\textcolor{\twocolour}{Reward}} + \underbrace{\sum_{o \in \mathcal{O}} \sum_{s' \in \mathcal{S}} \overbrace{G(o \mid a, s')}^\text{\makebox[0pt]{\textcolor{\twocolour}{Observation Probability}}} \sum_{s \in \mathcal{S}} \xi(s) \underbrace{P(s' \mid s, a)}_\text{\makebox[0pt]{\textcolor{\twocolour}{State Probability}}} \int_0^\infty \overbrace{{\e}^{-\beta \tau}}^\text{\makebox[0pt]{\textcolor{\twocolour}{Exponential Discount}}} \underbrace{F(\d \tau \mid s,a,s')}_\text{\makebox[0pt]{\textcolor{\twocolour}{Sojourn Time Probability}}} \overbrace{V^*\big(\xi(\cdot \mid a,\tau, o)\big)}^\text{\makebox[0pt]{\textcolor{\twocolour}{Next State Value}}}\Big]}_\text{\textcolor{\twocolour}{Discounted Expected Future Rewards}}}
\label{eq:POSMDP:discreteStatediscreteAction:valueFunctionOptimal:PF}
\end{equation}
or equivalently using Eq.~\eqref{eq:POSMDP:Q-PF},
\begin{equation}{\footnotesize
V^*(\xi) = \max_{a \in \mathcal{A}(\xi)}\Big[\underbrace{\sum_{s \in \mathcal{S}} \xi(s) R(s, a)}_\text{\textcolor{\twocolour}{Reward}} + \underbrace{\sum_{o \in \mathcal{O}} \sum_{s' \in \mathcal{S}} \overbrace{G(o \mid a, s')}^\text{\makebox[0pt]{\textcolor{\twocolour}{Observation Probability}}} \sum_{s \in \mathcal{S}} \xi(s) \int_0^\infty \overbrace{{\e}^{-\beta \tau}}^\text{\makebox[0pt]{\textcolor{\twocolour}{Exponential Discount}}} \underbrace{Q(\d \tau, s' \mid s, a)}_\text{\makebox[0pt]{\textcolor{\twocolour}{Joint State sojourn-time Probability}}} \overbrace{V^*\big(\xi(\cdot \mid a,\tau, o)\big)}^\text{\makebox[0pt]{\textcolor{\twocolour}{Next State Value}}}\Big]}_\text{\textcolor{\twocolour}{Discounted Expected Future Rewards}}.}
\label{eq:POSMDP:discreteStatediscreteAction:valueFunctionOptimal:Q}
\end{equation}

\subsection{Alpha vectors}
Since the belief simplex $\triangle$ is continuous, the optimal value function $V^*(\xi)$ in Eq.~\eqref{eq:POMDP:discreteStatediscreteAction:valueFunctionOptimal} is computationally intractable to calculate pointwise for every belief state $\xi \in \triangle$. To circumvent this, \citet{Sondik1971} made the following assumption. Consider a two-state $\mathcal{S}=\{s_1, s_2\}$ POMDP. In Fig.~\ref{fig:POMDP:AlphaVectorLinearAssumption:axis}, we represent the belief state on the horizontal axis, while on the vertical axis will be the value function. The agent knows the value with absolute certainty when with absolute certainty it knows in which state it is; see Fig.~\ref{fig:POMDP:AlphaVectorLinearAssumption:values}.
\begin{assumption}
The value is linearly propositional to how certain the agent believes in which state it is.\footnote{While \citet{Sondik1971} did not explicitly state this in his thesis, this assumption is apparent and required to show that the finite state, finite action, finite observation POMDP is piecewise linear and convex.}
\end{assumption}
In other words, linear interpolation fills in the values between the two known values; see Fig.~\ref{fig:POMDP:AlphaVectorLinearAssumption:linear}. The line segment that connects the two values in Fig.~\ref{fig:POMDP:AlphaVectorLinearAssumption:linear} is an example of what \citet{Sondik1971} called an $\alpha$-vector.
\begin{figure}[h!t]
\centering
\subfigure[Beginning to graph the value function over the belief simplex]{
\centering
\input{Fig/POMDPassumptionAlpha1.tex}
\label{fig:POMDP:AlphaVectorLinearAssumption:axis}
}\quad
\subfigure[The agent knows what the value is when it knows with absolute certainity in which state it is.]{
\centering
\input{Fig/POMDPassumptionAlpha2.tex}
\label{fig:POMDP:AlphaVectorLinearAssumption:values}
}\quad
\subfigure[The assumption is that the two values are connected by a linear function.]{
\centering
\input{Fig/POMDPassumptionAlpha3.tex}
\label{fig:POMDP:AlphaVectorLinearAssumption:linear}
}
\caption{An example of an $\alpha$-vector for a two-state $\mathcal{S} = \{s_1, s_2\}$ POMDP}
\label{fig:POMDP:AlphaVectorLinearAssumption}
\end{figure}
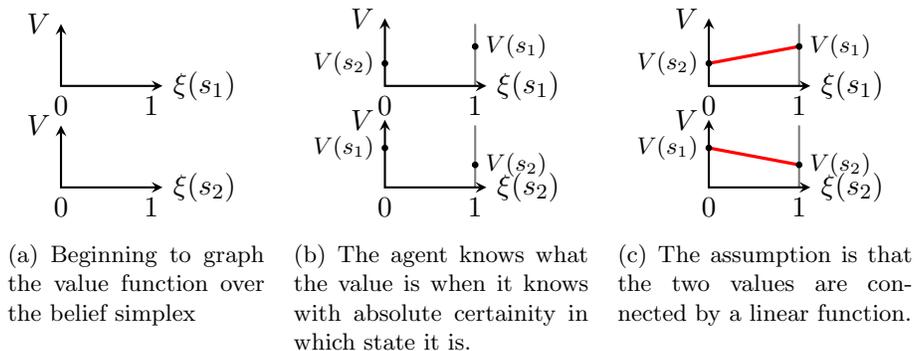
In general, we give the following definition.
\begin{definition}[$\alpha$-vector]
An $\alpha$-vector represents an $|\mathcal{S}|$-dimensional hyperplane that defines the value function for a particular action $a \in \mathcal{A}$ over a bounded region of the belief simplex. In this work, we represent $\alpha$-vectors in matrix form:
\begin{equation*}
\alpha = \begin{bmatrix} \alpha(s_1)\\ \alpha(s_2) \\ \vdots \\ \alpha(s_{|\mathcal{S}|})\end{bmatrix} = \begin{bmatrix} V(s_1) \\ V(s_2) \\ \vdots \\ V(s_{|\mathcal{S}|}) \end{bmatrix}.
\end{equation*}
\end{definition}

While we have generally referred to $V$ as the value function, we can also think of $V$ as a set of $\alpha$ vectors:
\begin{equation}
V = \{ \alpha_1, \alpha_2, \ldots, \alpha_m \}.
\end{equation}
To calculate the optimal value function\footnote{We will denote the optimal value function by $V^*$ and the set of $\alpha$ vectors by $V$.} at a particular belief state is 
\begin{align}
    V^*(\xi) &= \max_{\alpha \in V} \langle \xi, \alpha \rangle \label{eq:AlphaVector:OptimalV}\\
             &= \max_{\alpha \in V} \sum_{s \in \mathcal{S}} \xi(s) \alpha(s), \label{eq:AlphaVector:OptimalV:sum}
\end{align}
and the corresponding optimal policy $\pi^* \colon \xi \to a$ is
\begin{equation}\label{eq:POMDP:pi:optimal:xi,alpha}
\pi^* = \argmax_{i \in \{1, 2, \ldots, |V|\}} \left\langle \xi, \alpha_i\right\rangle.
\end{equation}

\subsection{Approximate Value Iteration}
A critical element to solve a POSMDP is to approximate the integral
\begin{equation}
\mu(\xi, s,o,a,s') = \int_0^\infty {\e}^{-\beta \tau} f(\tau \mid s,a,s') V^*\big(\xi(\cdot \mid a, \tau, o)\big) \d \tau. \label{eq:sojournTimeValueIntegral}
\end{equation}
We solve this problem using the Monte Carlo method with importance sampling described in \citet{Hammersley1964}.

Suppose we have a finite set $C$ whose elements $\tau_1, \tau_2, \ldots, \tau_{|C|}$ are independent sojourn time samples from the probability density function $D$. An unbiased estimator of $\mu(s,o,a,s')$ is
\begin{equation}
\hat{\mu}(\xi, s,o,a,s') = \frac{1}{|C|}\sum_{n = 1}^{|C|} {\e}^{-\beta \tau_n} \frac{f(\tau_n \mid s, a, s')}{D(\tau_n)}V^*\big(\xi(\cdot \mid a, \tau_n, o)\big)\label{eq:sojournTimeValueIntegral:MC}
\end{equation}
where the function $D(\tau_n)$ given by 
$$D(\tau_n) = \sum_{s \in \mathcal{S}} \sum_{a \in \mathcal{A}} \sum_{s' \in \mathcal{S}} w(s,a,s') f(\tau_n \mid s, a, s')$$ 
is a mixture distribution of each sojourn-time distribution for each $s,a,s'$ and the weights $w(s,a,s')$ reflect the proportion of samples that came from each distribution. Note that the weights must sum to one; that is, $$\sum_{s \in \mathcal{S}} \sum_{a \in \mathcal{A}} \sum_{s' \in \mathcal{S}} w(s,a,s') = 1.$$

By replacing the integral \eqref{eq:sojournTimeValueIntegral} with the approximation \eqref{eq:sojournTimeValueIntegral:MC} in the Bellman equation for POSMDP, we get
\begin{align}
V^*(\xi) &= \max_{a \in \mathcal{A}(\xi)}\left[R(\xi, a) + \sum_{o \in \mathcal{O}} \sum_{s' \in \mathcal{S}} G(o \mid a, s') \right. \notag\\
         &\qquad\qquad\quad \left. \sum_{s \in \mathcal{S}} \xi(s)  P(s' \mid s, a) \frac{1}{|C|}\sum_{n = 1}^{|C|} {\e}^{-\beta \tau_n} \frac{f(\tau_n \mid s, a, s')}{D(\tau_n)}V^*\big(\xi(\cdot \mid a, \tau_n, o)\big) \right].
\end{align}
We bring $\displaystyle\frac{1}{|C|} \sum_{n = 1}^{|C|} \frac{{\e}^{-\beta \tau_n}}{D(\tau_n)}$ to the beginning of the second term since it does not depend on state $s$ or observation $o$, yielding
\begin{align}
V^*(\xi) &= \max_{a \in \mathcal{A}(\xi)}\left[R(\xi,a) + \frac{1}{|C|} \sum_{n = 1}^{|C|}\frac{{\e}^{-\beta \tau_n}}{D(\tau_n)} \sum_{o \in \mathcal{O}} \sum_{s' \in \mathcal{S}} G(o \mid a, s') \right. \notag \\
         &\qquad\qquad\qquad \left. \sum_{s \in \mathcal{S}} \xi(s)  P(s' \mid s, a) f(\tau_n \mid s, a, s') V^*\big(\xi(\cdot \mid a, \tau_n, o)\big) \right]
\end{align}
and by Eq.~\eqref{eq:POSMDP:discreteStatediscreteAction:beliefState:POMDP2SMDP:denominator},
\begin{equation}
V^*(\xi) = \max_{a \in \mathcal{A}(\xi)}\left[ R(\xi,a) + \frac{1}{|C|} \sum_{n = 1}^{|C|}\frac{{\e}^{-\beta \tau_n}}{D(\tau_n)} \sum_{o \in \mathcal{O}} P(o \mid \xi, a, \tau_n) V^*\big(\xi(\cdot \mid a, \tau_n, o)\big) \right].\label{eq:POSMDP:discreteStatediscreteAction:approxValueFunction:P(o|xi,a,tau)}
\end{equation}

Next, we follow the procedure similar to POMDPs in \citet{Spaan2005a} to express the value function as a finite set of $\alpha$ vectors. Using Eq.~\eqref{eq:AlphaVector:OptimalV}, we can express the value of the next belief as an inner product as
\begin{equation*}
  V^*(\xi) = \max_{a \in \mathcal{A}(\xi)}\left[ R(\xi,a) + \frac{1}{|C|} \sum_{n = 1}^{|C|}\frac{{\e}^{-\beta \tau_n}}{D(\tau_n)} \sum_{o \in \mathcal{O}} P(o \mid \xi, a, \tau_n) \max_{\alpha \in V} \left\langle \xi(\cdot \mid a, \tau_n, o), \alpha \right\rangle \right].
\end{equation*}
Expressing the inner product as a sum by Eq.~\eqref{eq:AlphaVector:OptimalV:sum}, we have
\begin{equation*}
  V^*(\xi) = \max_{a \in \mathcal{A}(\xi)}\left[ R(\xi,a) + \frac{1}{|C|} \sum_{n = 1}^{|C|}\frac{{\e}^{-\beta \tau_n}}{D(\tau_n)} \sum_{o \in \mathcal{O}} P(o \mid \xi, a, \tau_n) \max_{\alpha \in V} \sum_{s' \in \mathcal{S}} \alpha(s') \xi(s' \mid a, \tau_n, o) \right]
\end{equation*}
and using Eq.~\eqref{eq:POSMDP:discreteStatediscreteAction:beliefState:POMDP2SMDP} for the value of the next belief,
\begin{align}
  V^*(\xi) &= \max_{a \in \mathcal{A}(\xi)} \left[R(\xi,a) + \frac{1}{|C|} \sum_{n = 1}^{|C|}\frac{{\e}^{-\beta \tau_n}}{D(\tau_n)} \sum_{o \in \mathcal{O}} \cancel{P(o \mid \xi, a, \tau_n)} \right. \notag \\
  &\qquad\qquad\qquad\qquad \left.\max_{\alpha \in V} \sum_{s' \in \mathcal{S}} \alpha(s') \frac{G(o \mid a, s') \displaystyle\sum_{s \in \mathcal{S}} \xi(s)P(s' \mid s, a) f(\tau_n \mid s, a, s') }{\cancel{P(o \mid \xi, a, \tau_n)}} \right]. \notag \\
    &= \max_{a \in \mathcal{A}(\xi)} \left[ R(\xi,a) + \frac{1}{|C|} \sum_{n = 1}^{|C|}\frac{{\e}^{-\beta \tau_n}}{D(\tau_n)} \sum_{o \in \mathcal{O}} \max_{\alpha \in V} \sum_{s' \in \mathcal{S}} \alpha(s') G(o \mid a, s') \right.\\
    &\qquad\qquad\qquad\qquad\qquad\qquad\qquad\qquad\left. \displaystyle\sum_{s \in \mathcal{S}} \xi(s)P(s' \mid s, a) f(\tau_n \mid s, a, s') \right] \notag \\
\intertext{Note how this allows the elimination of the computationally expensive term $P(o \mid \xi, a, \tau_n)$. Rearranging the order of the summations and the terms so that we have an $\alpha$-vector yields,}
	V^*(\xi) &= \max_{a \in \mathcal{A}(\xi)}\left[ R(\xi,a) + \frac{1}{|C|} \sum_{n = 1}^{|C|}\frac{{\e}^{-\beta \tau_n}}{D(\tau_n)} \sum_{o \in \mathcal{O}} \max_{\alpha \in V} \sum_{s \in \mathcal{S}} \xi(s) \right. \notag\\
	         &\qquad\qquad\qquad\qquad\qquad\qquad \left.\underbrace{\sum_{s' \in \mathcal{S}} \alpha(s') G(o \mid a, s')  P(s' \mid s, a) f(\tau_n \mid s, a, s')}_{\alpha(s \mid a, \tau_n, o)} \right] \notag \\
	&= \max_{a \in \mathcal{A}(\xi)}\left[ R(\xi,a) + \frac{1}{|C|} \sum_{n = 1}^{|C|}\frac{{\e}^{-\beta \tau_n}}{D(\tau_n)} \sum_{o \in \mathcal{O}} \max_{\alpha \in V} \sum_{s \in \mathcal{S}} \xi(s) \alpha(s \mid a, \tau_n, o) \right], \notag 
\end{align}
and finally by Eq.~\eqref{eq:AlphaVector:OptimalV},
\begin{equation}
	V^*(\xi) = \max_{a \in \mathcal{A}(\xi)} \left[ R(\xi,a) + \frac{1}{|C|} \sum_{n = 1}^{|C|}\frac{{\e}^{-\beta \tau_n}}{D(\tau_n)} \sum_{o \in \mathcal{O}} \max_{\alpha \in V} \left\langle \xi, \alpha(\cdot \mid a, \tau_n, o)\right\rangle \right],
\end{equation}
which is the approximate form of the POSMDP value function that we will use for the \textsc{ChronosPerseus} algorithm.

We can now write a concise backup operation that generates a new $\alpha$ vector for a specific belief $\xi$; that is,
\begin{equation}
\textsc{backup}(V, C, w, \xi) = \argmax_{\textcolor{\twocolour}{\alpha(\cdot \mid \xi, a)} \colon a \in \mathcal{A}, \alpha \in V} \left\langle \xi, \textcolor{\twocolour}{\alpha(\cdot \mid \xi, a)} \right\rangle \label{eq:POSMDP:discreteStatediscreteAction:beliefState:POSMDP2SMDP:backup}
\end{equation}
where for all $s \in \mathcal{S}$,
\begin{equation}
\textcolor{\twocolour}{\alpha(s \mid \xi, a)} = R(s,a) + \frac{1}{|C|} \sum_{n = 1}^{|C|}\frac{{\e}^{-\beta \tau_n}}{D(\tau_n)} \sum_{o \in \mathcal{O}} \argmax_{\textcolor{\threecolour}{\alpha(\cdot \mid a, \tau_n, o)} \colon \alpha \in V} \left\langle \xi, \textcolor{\threecolour}{\alpha(\cdot \mid a, \tau_n, o)} \right\rangle,
\label{eq:POSMDP:discreteStatediscreteAction:beliefState:POSMDP2SMDP:backup2}
\end{equation}
and
\begin{equation}
 \textcolor{\threecolour}{\alpha(s \mid a, \tau_n, o)} = \sum_{s' \in \mathcal{S}}  G(o \mid a, s') P(s' \mid s, a) f(\tau_n \mid s, a, s')\, \alpha(s').
\label{eq:POSMDP:discreteStatediscreteAction:beliefState:POSMDP2SMDP:backup3}
\end{equation}

\subsection{The Algorithm}

We present now the resulting contribution, the \textsc{ChronosPerseus} algorithm, which is an approximate point-based value iteration with importance sampling for POSMDPs. It proceeds in two steps as specified in Algorithm~\ref{alg:ChronosPerseus}: first, collect beliefs and sojourn time samples, and then approximate the value function using the backup operator.

The first step, as outlined in Algorithm~\ref{alg:ChronosPerseus:collectBeliefs}, is to let the agent explore the environment and collect a finite set of beliefs $B$, a finite set of sojourn times $C$, and the weights $w(s,a,s')$ that reflect the proportion of samples that came from each state-action-state transition; after these are collected, they remain fixed for the rest of the algorithm. In Step 2, as outlined in Algorithm~\ref{alg:ChronosPerseus:update}, from the initial value function $V_0$, the algorithm will continue to perform backups until a convergence criterion is achieved.

\begin{theorem}
Let $M = \min_{(s,a) \in \mathcal{K}} R(s,a)$. The initial value function $V_0$ is a single $\alpha$-vector with 
\begin{equation*}
    V_0(s) = \frac{M}{1 - \lambda} \qquad \forall s \in \mathcal{S},
\end{equation*}
where
\begin{equation*}
    \lambda = \begin{cases}
        \displaystyle\min_{1 \leq n \leq |C|} \frac{{\e}^{-\beta \tau_n}}{D(\tau_n)}, & \text{if $M \geq 0$}; \vspace{2mm}\\
        \displaystyle\max_{1 \leq n \leq |C|} \frac{{\e}^{-\beta \tau_n}}{D(\tau_n)}, & \text{if $M < 0$}.
    \end{cases}
\end{equation*}
\end{theorem}

\begin{proof}
Since $\mathcal{S}$ and $\mathcal{A}$ are finite sets, the per-stage minimum reward $M$ exists by Assumption~\ref{POSMDP:assumption:rewardBoundM}. 

If $M = 0$, then this is trivial: set $\alpha(s) = 0$ for all $s \in \mathcal{S}$. 

For $M > 0$, we start with Eq.~\eqref{eq:POSMDP:discreteStatediscreteAction:approxValueFunction:P(o|xi,a,tau)}, for any belief state $\xi$,
\begin{equation*}
V(\xi) = \max_{a \in \mathcal{A}(\xi)}\left[ R(\xi,a) + \frac{1}{|C|} \sum_{n = 1}^{|C|}\frac{{\e}^{-\beta \tau_n}}{D(\tau_n)} \sum_{o \in \mathcal{O}} P(o \mid \xi, a, \tau_n) V\big(\xi(\cdot \mid a, \tau_n, o)\big) \right].
\end{equation*}
Let $\lambda_n = \displaystyle\frac{{\e}^{-\beta \tau_n}}{D(\tau_n)}$, and $0 < \lambda < 1$. Since $M$ is positive, we would like the discount to be as small as possible. So, we define
\begin{equation*}
    \lambda = \min_{1 \leq n \leq |C|} \lambda_n, \qquad n^* = \argmin_{1 \leq n \leq |C|} \lambda_n, \qquad \text{and} \qquad \tau = \tau_{n*},
\end{equation*}
so that
\begin{align*}
    V(\xi) &\geq \max_{a \in \mathcal{A}(\xi)}\left[ R(\xi,a) + \frac{1}{|C|} \sum_{n = 1}^{|C|}\lambda \sum_{o \in \mathcal{O}} P(o \mid \xi, a, \tau) V\big(\xi(\cdot \mid a, \tau, o)\big) \right]\\
      &= \max_{a \in \mathcal{A}(\xi)}\left[ R(\xi,a) + \lambda \sum_{o \in \mathcal{O}} P(o \mid \xi, a, \tau) V\big(\xi(\cdot \mid a, \tau, o)\big) \right]\\
      &\geq \max_{a \in \mathcal{A}(\xi)}\left[ M + \lambda \left(\frac{M}{1 - \lambda}\right) \right]\\
      &= M + \frac{\lambda M}{1 - \lambda}\\
      &= \frac{M}{1 - \lambda} = V_0.
\end{align*}

If $M < 0$, then it is a negative reward, and so the discount should be as large as possible, so 
\begin{equation*}
    \lambda = \max_{1 \leq n \leq |C|} \lambda_n, \qquad n^* = \argmax_{1 \leq n \leq |C|} \lambda_n, \qquad \text{and} \qquad \tau = \tau_{n*}.
\end{equation*}
\end{proof}

\begin{center}
\begin{minipage}[t]{0.5\textwidth}
      \begin{algorithm}[H]
      \caption{\textsc{ChronosPerseus}}\label{alg:ChronosPerseus}
        \begin{algorithmic}[1]{\small
            \Function{ChronosPerseus}{$n$, $\xi_0$, $V_0$, $\epsilon$}
                    \State $(B, C, w) \gets$ \Call{collectBeliefs}{$n, \xi_0$}
                    \State $V' \gets V_0$
                    \Repeat
                    \State $V \gets V'$
                    \State $V' \gets$ \Call{update}{$B, C, w, V$}
                    \Until{$||V - V'||_{\infty, B} < \epsilon$}
                    \State \Return $V$
                \EndFunction}
            \end{algorithmic}
        \end{algorithm}
\end{minipage}
\end{center}

\begin{center}
\begin{minipage}[t]{\textwidth}
\begin{algorithm}[H]
\caption{\textsc{CollectBelief}}\label{alg:ChronosPerseus:collectBeliefs}
\begin{algorithmic}[1]
{\small
\Function{collectBeliefs}{$n, \xi_0$}
 \State $B \gets \{\xi_0\}$, $C \gets \varnothing$, $w \gets 0$ \Comment{Initialization}
 \Repeat
   \State Randomly select the belief state $\xi \in B$
	 \State{Generate state $s$ according to the multinomial distribution with weights $\xi$}
	 \State Randomly select an action $a \in \mathcal{A}$
	 \State Randomly select state $s'$ according to $P(\cdot \mid s, a)$
     \State{Generate a random sojourn-time $\tau$ according to  $f(\tau \mid s, a, s')$}
	 \State $C \gets C \cup \{\tau\}$, $w(s,a,s') \gets w(s,a,s') + 1$ \Comment{Add $\tau$ to set $C$ and update counts}
	 \State{Randomly select $o \in \mathcal{O}$ according to Eq.~\eqref{eq:POSMDP:discreteStatediscreteAction:beliefState:POMDP2SMDP:denominator}}
	 \State Update $\xi(\cdot \mid a, \tau, o)$ according to Eq.~\eqref{eq:POSMDP:discreteStatediscreteAction:beliefState:POMDP2SMDP}
	 \State $B \gets B \cup \{\xi(\cdot \mid a, \tau, o)\}$ \Comment{Add new belief to set $B$}
 \Until{$|B| = n$}
 \State $w \gets \displaystyle\frac{w}{||w||}$ \Comment{Normalize $w$; $||w|| = \displaystyle\sum_{s\in\mathcal{S}}\sum_{a \in \mathcal{A}}\sum_{s' \in \mathcal{S}} w(s,a,s')$.}
 \State \Return $B$, $C$, and $w$
\EndFunction
}
\end{algorithmic}
\end{algorithm}
\end{minipage} 
\end{center}

\begin{center}
\begin{minipage}[t]{0.6\textwidth}
\begin{algorithm}[H]
\caption{\textsc{Update}}\label{alg:ChronosPerseus:update}
\begin{algorithmic}[1]
{\small
\Function{update}{$B, C, w, V$}
   \State $B' \gets B$, $V' \gets \varnothing$ \Comment{Initialization}
	\While{$B' \neq \varnothing$}
	     \State Randomly select a belief $\xi \in B'$
		 \State $\alpha \gets$ \Call{backup}{$V, C, w, \xi$} \Comment{Eq \eqref{eq:POSMDP:discreteStatediscreteAction:beliefState:POSMDP2SMDP:backup}}
		 \If{$\langle \xi, \alpha \rangle < V(\xi)$} 
		 \State $\alpha \gets \displaystyle\argmax_{\alpha' \in V} \langle \xi, \alpha' \rangle$ 
	 \EndIf
     \State $B' \gets B' \smallsetminus \{\varsigma \in B' \mid \langle \varsigma, \alpha \rangle \geq V(\varsigma)\}$ 
	 \State $V' \gets V' \cup \{\alpha\}$ 
	 \EndWhile
	 \State $V \gets V'$
 \State \Return $V$
\EndFunction
}
\end{algorithmic}
\end{algorithm}
\end{minipage}
\end{center}

The complexity of computing Eq.~\eqref{eq:POSMDP:discreteStatediscreteAction:beliefState:POSMDP2SMDP:backup3} is $O(|\mathcal{S}|^2)$ since it needs to be calculated for every $(s,s')$ tuple, and it is done for every $\alpha \in V$, hence computing all $\alpha(\cdot \mid a, \tau, o)$ for every $a \in \mathcal{A}$, $\tau \in C$, and $o \in \mathcal{O}$ requires 
$O(|V| \times |\mathcal{S}|^2 \times |\mathcal{A}| \times |C| \times |\mathcal{O}|).$ The complexity of computing Eq.~\eqref{eq:POSMDP:discreteStatediscreteAction:beliefState:POSMDP2SMDP:backup2} requires the computation of all relevant $\alpha(\cdot \mid a, \tau, o)$, but then the summation and inner products require only $O(|\mathcal{S}| \times |C| \times |\mathcal{O}|)$ operations and another $O(|\mathcal{S}|)$ operations to add the reward (vector). Lastly, the \textsc{Backup} (Eq.~\eqref{eq:POSMDP:discreteStatediscreteAction:beliefState:POSMDP2SMDP:backup}) requires for all $\alpha(\cdot \mid \xi, a)$ another $O(|\mathcal{S}|)$ operations for the inner product.  Therefore, the complexity of the \textsc{Backup} operation requires
\begin{equation}
O(|V| \times |\mathcal{S}|^2 \times |\mathcal{A}| \times |C| \times |\mathcal{O}|).
\end{equation}
Given there are no more beliefs than observed sojourns time, $|B| \leq |C|$, the complexity of an \textsc{Update} iteration is therefore 
\begin{equation}
O(|V| \times |\mathcal{S}|^2 \times |\mathcal{A}| \times |C|^2 \times |\mathcal{O}|).
\end{equation}

\section{Applications}\label{sec:Applications}

We now apply \textsc{ChronosPerseus} to two interesting problems. The first is an episodic toy problem related to the often encountered situation in our daily life in which our waiting time influences our belief and, in turn, our decision. In this simple scenario, an agent is riding a bus at an unknown traffic intensity with the ability to give-up and ride its bike instead at any given bus stop. The second problem is a real application to a routine maintenance problem first presented with a continuous observation space by \citet{Zhang2017}.

These two examples will show how \textsc{ChronosPerseus} can fully solve real POSMDP problems with mixed observability, a mixture of observable continuous and discrete random sojourn times, discrete or continuous observation space, and episodic or non-episodic tasks.

\subsection{The Bus Problem}\label{sec:TheBusProblem}
We introduce the bus problem, an episodic toy example, that demonstrates \textsc{ChronosPerseus} follows our human intuition.

Suppose the agent boarded a bus and brought along a bicycle, which has been attached to the bus' rack. There are four bus stops until the agent reaches the final destination. There are three levels of traffic intensity $i$ (low $i = 1$, medium $i = 2$, and high $i = 3$) that remains the same throughout the agent's bus journey, but the traffic intensity $i$ is unknown to the agent.  When the bus is at stop $s$, the agent can either elect to stay on the bus, or to exit the bus and ride its bike directly to the destination (the last stop). Reaching the destination leads to a (lump sum) reward, but the longer it takes to reach it, the more discounted it will be. Suppose the agent decides to continue. In that case, the next state will be $s + 1$ with sojourn time $\tau$ according to the probability density function $f\big(\tau, (s+1, i) \mid (s, i)\big)$. If the agent chooses to take the bike, it takes a fixed time to reach the destination depending on the current stop $s$.

\begin{figure}[h!tbp]
\centering
\input{Fig/busProblem/busProblem.tex}
\caption{The agent waiting for the bus.}
\label{fig:example:BusProblem}
\end{figure}
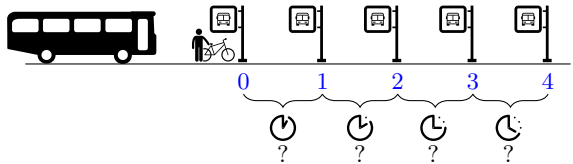

Intuitively, we would expect that the agent should ride the bus long enough to deduce which traffic intensity level: 
\begin{itemize}
    \item if the traffic intensity is low, then the agent would ride the bus to the end; \item if the traffic intensity is high, then the agent would get off at the next stop and ride the bike to the end; and
    \item if the traffic intensity is medium, then the agent would have to balance the advantage of taking the bus versus riding the bike.
\end{itemize}

With this intuition in mind, we can model this problem formally using the POSMDP framework. The state space consists of the bus stops and the three levels of traffic intensity, 
\begin{equation}
\mathcal{S} = \overbrace{\underbrace{\{ \textcolor{\threecolour}{0}, \textcolor{\threecolour}{1}, \textcolor{\threecolour}{2}, \textcolor{\threecolour}{3}, \textcolor{\threecolour}{4}\}}_\text{bus stops}}^\text{\textcolor{\threecolour}{observable}} \times \overbrace{\underbrace{\{ \textcolor{\twocolour}{1}, \textcolor{\twocolour}{2}, \textcolor{\twocolour}{3}\}}_\text{traffic}}^\text{\textcolor{\twocolour}{hidden}},
\end{equation}
which means the number of states is 
\begin{equation}
|\mathcal{S}| = 5 \times 3 = 15.
\end{equation}
A typical state would be the pair $(\textcolor{\threecolour}{s}, \textcolor{\twocolour}{i}) \in \mathcal{S}$, where the bus is at stop $\textcolor{\threecolour}{s}$ and the traffic intensity is $\textcolor{\twocolour}{i}$. This state is mixed observable since the agent can observe perfectly at which bus stop $\textcolor{\threecolour}{s}$ it is, while it cannot observe directly the traffic intensity $\textcolor{\twocolour}{i}$. We will assume that the traffic intensity remains $\textcolor{\twocolour}{i}$ throughout the entire bus ride; that is, $$(0,\textcolor{\twocolour}{i}) \to (1, \textcolor{\twocolour}{i}) \to (2, \textcolor{\twocolour}{i}) \to (3, \textcolor{\twocolour}{i}) \to (4, \textcolor{\twocolour}{i}).$$

The action space consists of two actions
\begin{equation}
\mathcal{A} = \{\text{bus}, \text{bike}\}.
\end{equation}

The observation space consists of the bus stops that the agent observes,
\begin{equation}
\mathcal{O} = \underbrace{\{ \textcolor{\threecolour}{0}, \textcolor{\threecolour}{1}, \textcolor{\threecolour}{2}, \textcolor{\threecolour}{3}, \textcolor{\threecolour}{4}\}}_\text{bus stops}.
\end{equation}
Please note that the observation space is not the same as the state space; in other words, $\mathcal{O} \neq \mathcal{S}$. Since only a component of the state is fully known to the agent, this problem is called mixed observable \citep{Ong2010}.  We will discuss this when we create the observation transition function.

If the agent decides to continue riding the bus, the traffic intensity $i$ does not change, but the bus stop changes from $s$ to $s + 1$. When the last stop is reached ($s=4$), the problem resets probabilistically by placing the agent at bus stop $s=0$ independently of the action selected, and randomly selecting a new traffic intensity $i$ with equal probability; this is reflected in the probability $P\big((0, i') \mid (4, i), \text{bus}\big) = \frac13$.

To illustrate this visually, a state transition diagram for the bus action is given in Fig.~\ref{fig:POSMDP:busProblem:BusAction}. Thus, the transition probability matrices for continuing to ride the bus are
\begin{align}
P\big((\textcolor{\threecolour}{\cdot}, i) \mid (\textcolor{\twocolour}{\cdot}, i), \text{bus}\big) &= \begin{bNiceMatrix}[first-col, first-row, code-for-first-col={\scriptstyle\color{\twocolour}}, code-for-first-row={\scriptstyle\color{\threecolour}}, nullify-dots] & s' = 0 & s' = 1 & s' = 2 & s'= 3 & s' = 4\\ 
                                                                 s = 0 & 0       & 1      & 0      & \Cdots & 0\\
                                                                 s = 1 & \Vdots  & \Ddots & \Ddots & \Ddots & \Vdots\\
                                                                 s = 2 &         &        &        &        & 0\\
                                                                 s = 3 & 0       & \Cdots &        & 0      & 1\\
                                                                 s = 4 & \frac13 & 0      & \Cdots &        & 0\\
                                                         \end{bNiceMatrix}\\
\intertext{and}
P\big((\textcolor{\threecolour}{\cdot}, i' \neq i) \mid (\textcolor{\twocolour}{\cdot}, i), \text{bus}\big) &= \begin{bNiceMatrix}[first-col, first-row, code-for-first-col={\scriptstyle\color{\twocolour}}, code-for-first-row={\scriptstyle\color{\threecolour}}, nullify-dots] & s' = 0 & s' = 1 & s' = 2 & s'= 3 & s' = 4\\
                                                                 s = 0 & 0       & \Cdots &        &         & 0\\
                                                                 s = 1 & \Vdots  &        &        &         & \Vdots\\
                                                                 s = 2 &         &        &        &         &  \\
                                                                 s = 3 & 0       & \Cdots &        &         & 0 \\
                                                                 s = 4 & \frac13 & 0      & \Cdots &         & 0\\
                                                         \end{bNiceMatrix}
\end{align}
for all traffic intensities $i \in \{1, 2, 3\}$.
If at stop $s \neq 4$ the agent decides to stop riding the bus and ride the bike to the final stop, the traffic intensity $i$ does not change, but the bus stop changes directly from $s$ to $s = 4$. Again, when the last stop is reached, the problem resets probabilistically by placing the agent at bus stop $0$ independently of the selected action, and randomly selecting a new traffic intensity $i$ with equal probability. Visually, this can be illustrated with a state diagram for the bike action, which is given in Fig.~\ref{fig:POSMDP:busProblem:BikeAction}.  Thus, the transition probability matrices for getting off the bus and riding the bike are
\begin{align}
P\big((\textcolor{\threecolour}{\cdot}, i) \mid (\textcolor{\twocolour}{\cdot}, i), \text{bike}\big) &= \begin{bNiceMatrix}[first-col, first-row, code-for-first-col={\scriptstyle\color{\twocolour}}, code-for-first-row={\scriptstyle\color{\threecolour}}, nullify-dots] & s' = 0 & s' = 1 & s' = 2 & s'= 3 & s' = 4\\
                                                                 s = 0 & 0       & \Cdots &        & 0       & 1\\
                                                                 s = 1 & \Vdots  &        &        & \Vdots  & \Vdots\\
                                                                 s = 2 &         &        &        &         &  \\
                                                                 s = 3 & 0       & \Cdots &        & 0       & 1 \\
                                                                 s = 4 & \frac13 & 0      & \Cdots &         & 0\\
                                                         \end{bNiceMatrix}\\
\intertext{and}
P\big((\textcolor{\threecolour}{\cdot}, i' \neq i) \mid (\textcolor{\twocolour}{\cdot}, i), \text{bike}\big) &= \begin{bNiceMatrix}[first-col, first-row, code-for-first-col={\scriptstyle\color{\twocolour}}, code-for-first-row={\scriptstyle\color{\threecolour}}, nullify-dots] & s' = 0 & s' = 1 & s' = 2 & s'= 3 & s' = 4\\
                                                                 s = 0 & 0       & \Cdots &        &         & 0\\
                                                                 s = 1 & \Vdots  &        &        &         & \Vdots\\
                                                                 s = 2 &         &        &        &         &  \\
                                                                 s = 3 & 0       & \Cdots &        &         & 0 \\
                                                                 s = 4 & \frac13 & 0      & \Cdots &         & 0\\
                                                         \end{bNiceMatrix}
\end{align}
for all traffic intensities $i \in \{1, 2, 3\}$. 

\begin{figure}[h!tbp]
\centering
\input{Fig/busProblemBus.tex}
\caption{Bus problem state transition diagram when the agent selects action $a =$ bus. The problem begin at stop $s=0$, with one of the randomly selected initial intensity $i$. The bus action allows the agent to go from one bus stop to the next while traffic intensity remains constant, until it reaches the end of the bus line. A the end, it is returned to the beginning of the bus line for a new randomly selected intensity.}
\label{fig:POSMDP:busProblem:BusAction}
\end{figure}
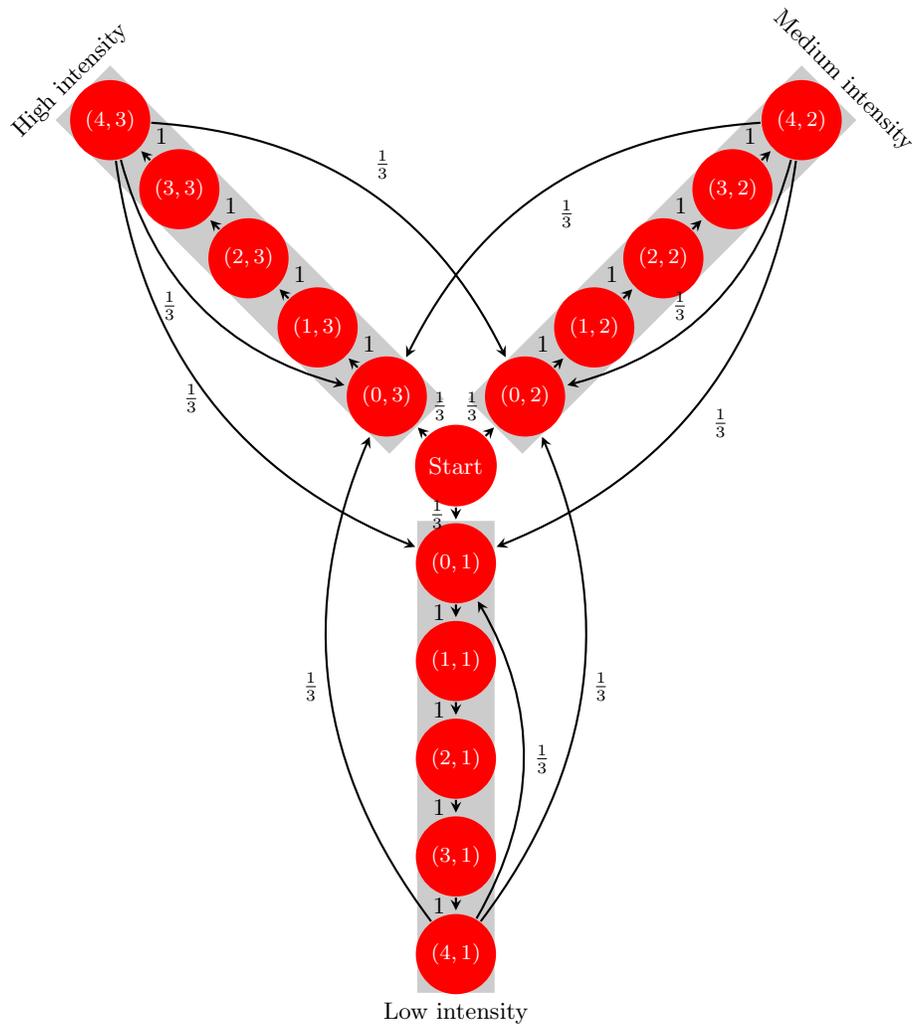

\begin{figure}[h!tbp]
\centering
\input{Fig/busProblemBike.tex}
\caption{Bus problem state transition diagram when the agent selects action $a =$ bike. Riding the bike leads directly to the last stop, at which point, it is returned to the beginning of the bus line for a new randomly selected intensity.}
\label{fig:POSMDP:busProblem:BikeAction}
\end{figure}
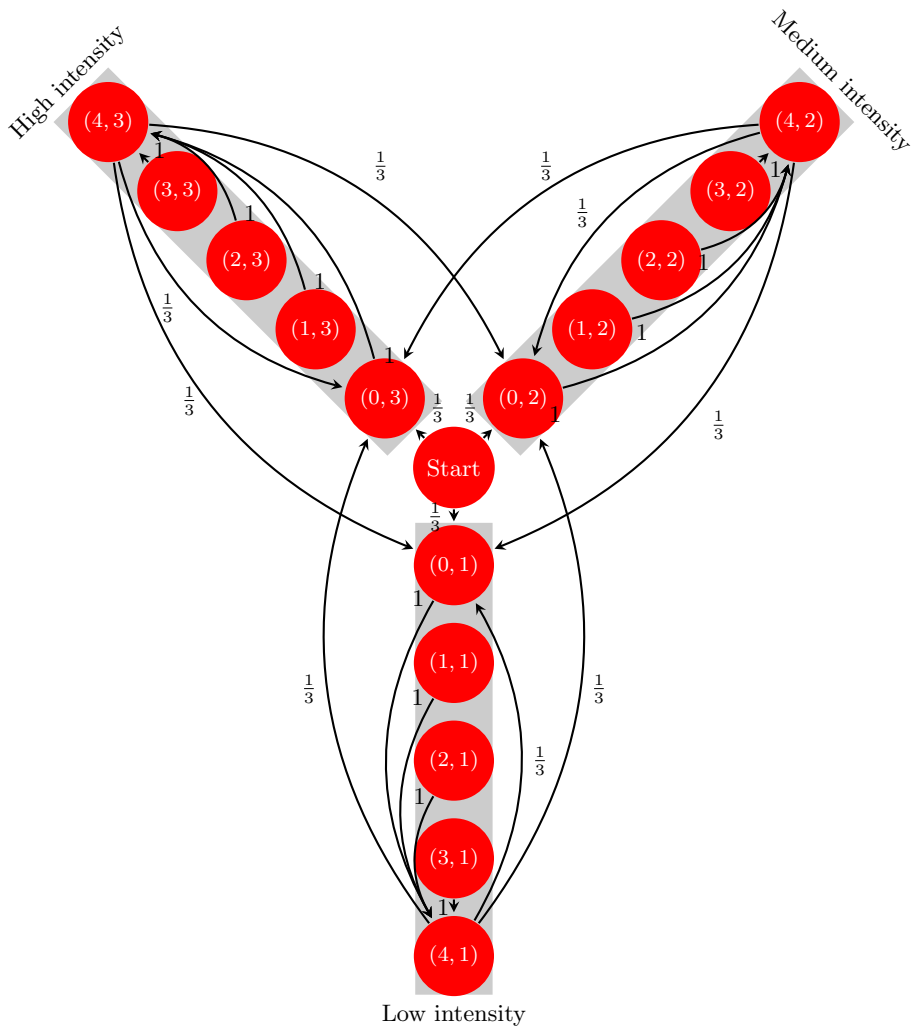

If the agent continues to ride the bus, the sojourn time distributions $f\big(\tau \mid (s,i), \text{bus}, (s', i)\big)$ follow an inverse Gaussian probability density distribution \citep{Tweedie1957}
\begin{equation} \label{eq:pdf:InverseGaussian}
f(x \mid \mu, \lambda) = \sqrt{\frac{\lambda}{2\pi x^3}}\exp\left[-\frac{\lambda(x - \mu)^2}{2\mu^2 x}\right], \qquad \forall x > 0, \qquad \mu > 0, \ \lambda > 0,
\end{equation}
where $\mu$ is the expected value and the variance is $\mu^3 / \lambda$, the parameters conditional on the bus stops and traffic intensity given by
\begin{align}
\mu\big((\textcolor{\twocolour}{\cdot}, 1),\text{bus}, (\textcolor{\threecolour}{\cdot},1)\big) &= \begin{bNiceMatrix}[first-col, first-row, code-for-first-col={\scriptstyle\color{\twocolour}}, code-for-first-row={\scriptstyle\color{\threecolour}}] & s' = 0 & s' = 1 & s' = 2 & s'= 3 & s' = 4\\
                                                                 s = 0 & 0   & 5 & 0 & 0 & 0\\
                                                                 s = 1 & 0   & 0 & 5 & 0 & 0\\
                                                                 s = 2 & 0   & 0 & 0 & 5 & 0\\
                                                                 s = 3 & 0   & 0 & 0 & 0 & 5\\
                                                                 s = 4 & 455 & 0 & 0 & 0 & 0\\
                                                         \end{bNiceMatrix} &\text{\shortstack[r]{($i = 1$)\\low intensity}}\label{eq:POSMDP:example:bus:muLow}\\
\mu\big((\textcolor{\twocolour}{\cdot}, 2),\text{bus}, (\textcolor{\threecolour}{\cdot}, 2)\big) & = \begin{bNiceMatrix}[first-col, first-row, code-for-first-col={\scriptstyle\color{\twocolour}}, code-for-first-row={\scriptstyle\color{\threecolour}}] & s' = 0 & s' = 1 & s' = 2 & s'= 3 & s' = 4\\
                                                                 s = 0 & 0   & 5 & 0  & 0  & 0 \\
                                                                 s = 1 & 0   & 0 & 10 & 0  & 0 \\
                                                                 s = 2 & 0   & 0 & 0  & 10 & 0 \\
                                                                 s = 3 & 0   & 0 & 0  & 0  & 20\\
                                                                 s = 4 & 455 & 0 & 0  & 0  & 0 \\
                                                         \end{bNiceMatrix} &\text{\shortstack[r]{($i = 2$)\\ medium intensity}}\label{eq:POSMDP:example:bus:muMedium}\\
\mu\big((\textcolor{\twocolour}{\cdot}, 3),\text{bus}, (\textcolor{\threecolour}{\cdot}, 3)\big) &= \begin{bNiceMatrix}[first-col, first-row, code-for-first-col={\scriptstyle\color{\twocolour}}, code-for-first-row={\scriptstyle\color{\threecolour}}] & s' = 0 & s' = 1 & s' = 2 & s'= 3 & s' = 4\\
                                                                 s = 0 & 0   & 10 & 0  & 0  & 0\\
                                                                 s = 1 & 0   & 0  & 25 & 0  & 0\\
                                                                 s = 2 & 0   & 0  & 0  & 25 & 0\\
                                                                 s = 3 & 0   & 0  & 0  & 0  & 45\\
                                                                 s = 4 & 455 & 0  & 0  & 0  & 0\\
                                                         \end{bNiceMatrix} &\text{\shortstack[r]{($i = 3$)\\high intensity}}\label{eq:POSMDP:example:bus:muHigh}\\
\mu\big((s, i), \text{bus}, (s', i' \neq i)\big) &= 0
\end{align}
and
\begin{equation}
\lambda\big((s,i),\text{bus},(s',i)\big) = 10 \cdot \mu\big((s,i),\text{bus},(s',i)\big)^2.
\end{equation}
The graphs of the sojourn time distributions for riding the bus from stop to stop at each traffic intensity are shown in Fig.~\ref{fig:POSMDP:busproblem:sojournTimeDistributions}. We can see that the time it takes from stop to stop in low traffic is relatively short compared to the long times between stops in high traffic; the mean total time is $20$ in low traffic whereas the mean total time in high traffic is $105$.
Note that technically we cannot have $\mu=0$ in the matrices because the inverse Gaussian distribution is defined for $\tau > 0$. This is not an issue since the corresponding probability $P(s' \mid s,a) = 0$, hence, $Q(\tau, s' \mid s, a) = 0$. 

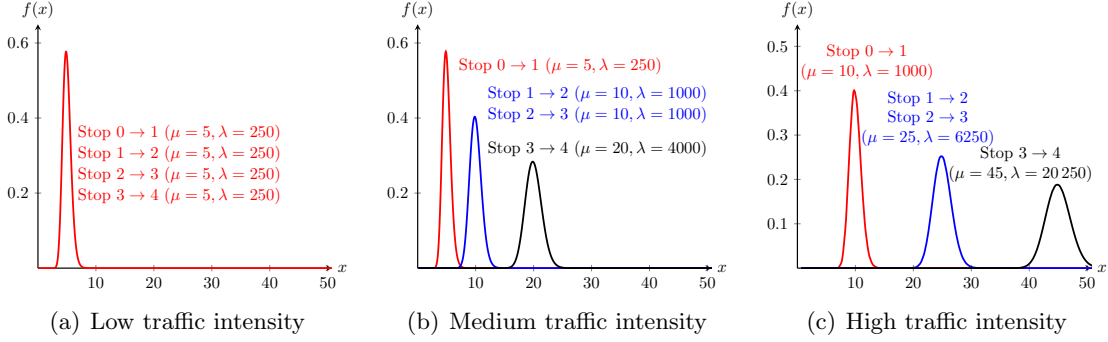
\begin{figure}
    \centering
    \subfigure[Low traffic intensity]{
\begin{tikzpicture}[scale=0.57]
\pgfmathdeclarefunction{invGauss}{2}{%
  \pgfmathparse{sqrt(#2/(2*pi*x^3))*exp(-(#2*(x-#1)^2)/(2*#1^2*x))}%
}
\begin{axis}[every axis plot post/.append style={
  mark=none,
  samples=500,
  smooth, 
  >=stealth, 
  very thick,
  domain=0:51}, 
  axis lines = middle,
  xlabel=$x$,
  ylabel=$f(x)$,
  every axis x label/.style={
    at={(ticklabel* cs:1.0)},
    anchor=west,
},
every axis y label/.style={
    at={(ticklabel* cs:1.0)},
    anchor=south,
},
  thick,
  xmin = 0,
  xmax = 50.7,
  ymin = 0,
  ymax = 0.65] 
  \addplot[color=\twocolour] {invGauss(5,250)} node[below right] at (axis cs:5.8,0.4){\shortstack{ Stop $0 \to 1$ ($\mu = 5, \lambda = 250$)\\Stop $1 \to 2$ ($\mu = 5, \lambda = 250$)\\Stop $2 \to 3$ ($\mu = 5, \lambda = 250$)\\Stop $3 \to 4$ ($\mu = 5, \lambda = 250$)}};
\end{axis}
\end{tikzpicture}
    }
    \subfigure[Medium traffic intensity]{
\begin{tikzpicture}[scale=0.57]
\pgfmathdeclarefunction{invGauss}{2}{%
  \pgfmathparse{sqrt(#2/(2*pi*x^3))*exp(-(#2*(x-#1)^2)/(2*#1^2*x))}%
}
\begin{axis}[every axis plot post/.append style={
  mark=none,
  samples=400,
  smooth, 
  >=stealth, 
  very thick,
  domain=0:51},  
  axis lines = middle,
  xlabel=$x$,
  ylabel=$f(x)$,
  every axis x label/.style={
    at={(ticklabel* cs:1.0)},
    anchor=west,
},
every axis y label/.style={
    at={(ticklabel* cs:1.0)},
    anchor=south,
},
  thick,
  xmin = 0,
  xmax = 50.7,
  ymin = 0,
  ymax = 0.65] 
  \addplot[color=\twocolour] {invGauss(5,250)} node[above right] at (axis cs:6,0.5){Stop $0 \to 1$ ($\mu = 5, \lambda = 250$)};
  \addplot[color=\threecolour] {invGauss(10,1000)} node[above right] at (axis cs:11,0.37){\shortstack{Stop $1 \to 2$ ($\mu = 10, \lambda = 1000$) \\ Stop $2 \to 3$ ($\mu = 10, \lambda = 1000$)}};
  \addplot[color=black] {invGauss(20,4000)} node[above right] at (axis cs:11,0.28){Stop $3 \to 4$ ($\mu = 20, \lambda = 4000$)};
\end{axis}
\end{tikzpicture}
    
    }
    \subfigure[High traffic intensity]{
\begin{tikzpicture}[scale=0.57]
\pgfmathdeclarefunction{invGauss}{2}{%
  \pgfmathparse{sqrt(#2/(2*pi*x^3))*exp(-(#2*(x-#1)^2)/(2*#1^2*x))}%
}
\begin{axis}[every axis plot post/.append style={
  mark=none,
  samples=80,
  smooth, 
  >=stealth, 
  very thick,
  domain=0:51},  
  axis lines = middle,
  xlabel=$x$,
  ylabel=$f(x)$,
  every axis x label/.style={
    at={(ticklabel* cs:1.0)},
    anchor=west,
},
every axis y label/.style={
    at={(ticklabel* cs:1.0)},
    anchor=south,
},
  thick,
  xmin = 0,
  xmax = 50.7,
  ymin = 0,
  ymax = 0.55] 
  \addplot[color=\twocolour] {invGauss(10,1000)} node[above] at (axis cs:12,0.41){\shortstack{Stop $0 \to 1$\\ ($\mu = 10, \lambda = 1000$)}};
  \addplot[color=\threecolour] {invGauss(25,6250)} node[above] at (axis cs:22.5,0.26){\shortstack{Stop $1 \to 2$\\ Stop $2 \to  3$\\ ($\mu = 25, \lambda = 6250$)}};
  \addplot[color=black] {invGauss(45,20250)} node[above] at (axis cs:38.5,0.18){\shortstack{Stop $3 \to 4$\\($\mu = 45, \lambda = 20\,250$)}};
\end{axis}
\end{tikzpicture}
    }
     \caption{The sojourn time distributions for the action $a =$ bus.}
    \label{fig:POSMDP:busproblem:sojournTimeDistributions}
\end{figure}

If the agent decides to stop riding the bus and use its bicycle, then the sojourn time distribution follows a deterministic function $f\big(\tau \mid (s,i), \text{bike}, (s', i)\big)$. If $c_0$ is a constant, then the probability mass function is
\begin{equation}
    f\left(\tau \mid c_0\big((s,i), \text{bike}, (s',i)\big)\right) = \begin{cases}
           1, & \text{if } \tau = c_0\big((s,i), \text{bike}, (s',i)\big);\\
           0, &\text{otherwise},
           \end{cases}
\end{equation}
where
\begin{align}
c_0\big((\textcolor{\twocolour}{\cdot},i), \text{bike}, (\textcolor{\threecolour}{\cdot},i)\big) &= \begin{bNiceMatrix}[first-col, first-row, code-for-first-col={\scriptstyle\color{\twocolour}}, code-for-first-row={\scriptstyle\color{\threecolour}}, nullify-dots] & s' = 0 & s' = 1 & s' = 2 & s'= 3 & s' = 4\\
                                                         s = 0 & 0      & \Cdots &         & 0      & 30\\
                                                         s = 1 & \Vdots &        &         & \Vdots & 25\\
                                                         s = 2 &        &        &         &        & 20\\
                                                         s = 3 & 0      & \Cdots &         & 0      & 12\\
                                                         s = 4 & 455    & 0      & \Cdots  &        & 0
                                        \end{bNiceMatrix} \label{eq:POSMDP:example:bus:c0}
\intertext{and}
c_0\big((s,i), \text{bike}, (s',i'\neq i)\big) &= 0.
\end{align}

Since \textsc{ChronosPerseus} uses an infinite planning horizon, we need to ensure rewards from one trip (episode) from $s=0$ to $s=4$ does not the affect the value for the next trip (episode) by creating a large discount when the environment resets probablistically. This technique turns this episodic task into a non-episodic task, allowing \textsc{ChronosPerseus} to work seamlessly for episodic and non episodic tasks. We assume continuous-time discounting at a rate of $\beta = 0.02$. This means that the present value of one reward unit received at time $\tau$ units in the future equals ${\e}^{-0.02 \tau} = \gamma$. To get a discount factor of $\gamma = 0.0001$, we set $\tau = 455$ between $s = 4$ to $s = 0$; this is the entry for $\mu\big( (4,\cdot), \text{bus}, (0, \cdot) \big)$ in \eqref{eq:POSMDP:example:bus:muLow}--\eqref{eq:POSMDP:example:bus:muHigh} and $c_0\big( (4, \cdot), \text{bike}, (0, \cdot)\big)$ in \eqref{eq:POSMDP:example:bus:c0}.

The cumulative distribution function is 
\begin{equation}
F\left(\tau \mid c_0\big((s,i), \text{bike}, (s',i)\big)\right) = \begin{cases}
              0, & \text{if } \tau < c_0\big((s,i), \text{bike}, (s',i)\big);\\
              1, &\text{if } \tau \geq c_0\big((s,i), \text{bike}, (s',i)\big).
              \end{cases}
\end{equation}

The sojourn time-state transition function is defined by
\begin{align}
Q\big(\tau, (s', i) \mid (s,i), a \big) &= P\big((s', i) \mid (s,i), a\big) F\big(\tau \mid (s,i), a, (s',i) \big)\\
 &= P\big((s', i) \mid (s,i), a\big) \int_0^\tau f\big(t \mid (s,i), a, (s',i) \big)\, \mathrm{d} t
\end{align}

Given the landing state $(s', i)$, regardless of which action $a$ was selected, the agent with certainty would observe bus stop $s'$. Remember, it is the traffic intensity $i$ that is hidden from the agent, but this does not influence what bus stop $s'$ is observed; it is deterministic what bus stop is observed depending on whether the agent decides to continue to ride the bus or take the bicycle. This leads to the observation probability matrix given by
\begin{equation}
G\big(\textcolor{\threecolour}{\cdot} \mid a, (\textcolor{\twocolour}{\cdot},i)\big) = \begin{bNiceMatrix}[first-col, first-row, code-for-first-col={\scriptstyle\color{\twocolour}}, code-for-first-row={\scriptstyle\color{\threecolour}}, nullify-dots] & o = 0 & o = 1 & o = 2 & o = 3 & o = 4\\
                                                   s' = 0 & 1       & 0      & \Cdots &   & 0     \\
                                                   s' = 1 & 0       & \Ddots & \Ddots &   & \Vdots\\
                                                   s' = 2 & \Vdots  & \Ddots &        &   &       \\
                                                   s' = 3 &         &        &        &   & 0     \\
                                                   s' = 4 & 0       & \Cdots &        & 0 & 1\\
                          \end{bNiceMatrix} \qquad \forall i \in \{1, 2, 3\}.
\end{equation}

Now, we will setup the reward function. Set the lump sum reward $r_1\big((s,i),a\big) = 0$ and continuous reward rate as $r_2\big((s,i),a,(s',i')\big) = -1$. Using the Laplace transform of the inverse Gaussian distribution 
\begin{equation}
    \int_0^\infty {\e}^{-\beta \tau} f\big(\tau \mid (s,i), a, (s',i')\big) \, \d \tau,
\end{equation}
the reward function is 
\begin{equation}
R\big((s,i),a\big) = -\frac{1}{\beta}\sum_{(s',\, i') \in \mathcal{S}} P\big((s', i') \mid (s,i), a\big) \Big(1 - M\big((s,i),a,(s',i')\big)\Big)
\end{equation}
where
\begin{equation}
M\big((s,i),a,(s',i')\big) = \begin{cases}
         \exp\left[-\beta c_0\big((s,i), \text{bike}, (s',i)\big)\right], & \text{if $a =$ bike};\\
	     \exp \left[ \frac{\lambda(s, \text{bus}, s')}{\mu(s, \text{bus}, s')} \left(1 - \sqrt{1 + \frac{2 \mu(s, \text{bus}, s')^2 \beta}{\lambda(s, \text{bus}, s')}}\right)\right], & \text{if $a =$ bus}.
   \end{cases}
\end{equation}

The agent begins at bus stop $0$, and since this is fully observable, but the traffic level intensity is not, the agent assumes that the three traffic intensities are equally likely; that is,
\begin{equation}
\xi_0\big((\textcolor{\twocolour}{\cdot}, i)\big) = \begin{bNiceMatrix}[first-col, code-for-first-col={\scriptstyle\color{\twocolour}}] s = 0 & \frac13 \\ s = 1 & 0 \\ s = 2 & 0 \\ s = 3 & 0 \\ s = 4 & 0  \end{bNiceMatrix} \qquad \forall i \in \{1, 2, 3\}.
\end{equation}

The lump sum reward is 
\begin{equation}
    r_1\big( (4, i), a \big) = \begin{cases}
              100, & \forall i \in \{1, 2, 3\}, a \in \mathcal{A};\\
              0,   & \text{otherwise,}
              \end{cases}
\end{equation}
and the continuous reward rate is 
\begin{equation}
    r_2(\cdot, \cdot, \cdot) = 0.
\end{equation}

\noindent \textbf{Results.} The resulting policy is shown on Fig.~\ref{fig:POSMDP:example:bus:optimalpolicy}. Fig.~\ref{fig:POSMDP:example:bus:optimalpolicy:stop0} shows that if the agent had enough evidence to believe the traffic was medium or high, it should ride its bike (expected time 30) rather than staying on the bus (expected time of 45 and 105 respectively). But in lack of any information initially, the best option is to stay on the bus for one stop (for 5 time units) and then decide (the exact center of the simplex would be a red dot for the bus action). After one stop (Fig.~\ref{fig:POSMDP:example:bus:optimalpolicy:stop1}, medium and high traffic already have significant differences in remaining expected time, making the bike a better option only in the high traffic condition. The expected remaining time are 25 for the bike, and 15, 40 and 95 for the low, medium, and high traffic respectively. But the bike action must be compared against the  averaged of the 3 traffic intensity's expected discounting weighted by the agent's belief. In Fig.~\ref{fig:POSMDP:example:bus:optimalpolicy:stop2} and Fig.~\ref{fig:POSMDP:example:bus:optimalpolicy:stop3} similarly shows the optimal action at stops 2 and 3 given the current belief about the traffic intensities. Note that because we are looking at belief-weighted exponential discounting over stochastic inverse-Gaussian distributed sojourn time, the action decision boundary may not look as it would if we instead assumed a negative reward per time unit and no discounting.

\begin{figure}
\subfigure[Stop 0]{
\begin{tikzpicture}[>=stealth]
\begin{axis}[width=0.65*\textwidth,
                axis lines=middle,
                axis equal,
                xlabel=low,
                ylabel=medium,
                zlabel=high,
                xmin=0,
                xmax=1.25,
                ymin=0,
                ymax=1.25,
                zmin=0,
                zmax=2,
                xtick={0,1},
                ytick={0,1},
                ztick={0,1},
                view={135}{30}, thick]
\addplot3[patch, color=\twocolour!50, fill opacity=0.75, faceted color=black, very thick] coordinates{(1,0,0) (0,1,0) (0,0,1)};
\addplot3[mark=*] coordinates {(1,0,0)} node (s1') {};
\addplot3[mark=*] coordinates {(0,1,0)} node (s2') {};
\addplot3[mark=*] coordinates {(0,0,1)} node (s3') {};
\addplot3[mark=*, color = \twocolour, mark size = 1pt, only marks] table [x=i1, y=i2, z=i3, col sep=comma] {PolicyStop0Action0.csv};
\addplot3[mark=*, color = \threecolour, mark size = 1pt, only marks] table [x=i1, y=i2, z=i3, col sep=comma] {PolicyStop0Action1.csv};
\node[above left, xshift=-2ex, yshift=2ex] (s1) at (1,0,0) {\footnotesize $\begin{bmatrix}1\\0\\0\end{bmatrix}$};
\node[above right, xshift=2ex, yshift=2ex] (s2) at (0,1,0) {\footnotesize $\begin{bmatrix}0\\1\\0\end{bmatrix}$};
\node[right, xshift=3ex] (s3) at (0,0,1) {\footnotesize $\begin{bmatrix}0\\0\\1\end{bmatrix}$};
\path[->] (s1) edge[bend left] (s1');
\path[->] (s2) edge[bend right] (s2');
\path[->] (s3) edge[bend right] (s3');
\end{axis}
\end{tikzpicture}
\label{fig:POSMDP:example:bus:optimalpolicy:stop0}
}
\subfigure[Stop 1]{
\begin{tikzpicture}[>=stealth]
\begin{axis}[width=0.65*\textwidth,
                axis lines=middle,
                axis equal,
                xlabel=low,
                ylabel=medium,
                zlabel=high,
                xmin=0,
                xmax=1.25,
                ymin=0,
                ymax=1.25,
                zmin=0,
                zmax=2,
                xtick={0,1},
                ytick={0,1},
                ztick={0,1},
                view={135}{30}, thick]
\addplot3[patch, color=\twocolour!50, fill opacity=0.75, faceted color=black, very thick] coordinates{(1,0,0) (0,1,0) (0,0,1)};
\addplot3[mark=*] coordinates {(1,0,0)} node (s1') {};
\addplot3[mark=*] coordinates {(0,1,0)} node (s2') {};
\addplot3[mark=*] coordinates {(0,0,1)} node (s3') {};
\addplot3[mark=*, color = \twocolour, mark size = 1pt, only marks] table [x=i1, y=i2, z=i3, col sep=comma] {PolicyStop1Action0.csv};
\addplot3[mark=*, color = \threecolour, mark size = 1pt, only marks] table [x=i1, y=i2, z=i3, col sep=comma] {PolicyStop1Action1.csv};
\node[above left, xshift=-2ex, yshift=2ex] (s1) at (1,0,0) {\footnotesize $\begin{bmatrix}1\\0\\0\end{bmatrix}$};
\node[above right, xshift=2ex, yshift=2ex] (s2) at (0,1,0) {\footnotesize $\begin{bmatrix}0\\1\\0\end{bmatrix}$};
\node[right, xshift=3ex] (s3) at (0,0,1) {\footnotesize $\begin{bmatrix}0\\0\\1\end{bmatrix}$};
\path[->] (s1) edge[bend left] (s1');
\path[->] (s2) edge[bend right] (s2');
\path[->] (s3) edge[bend right] (s3');
\end{axis}
\end{tikzpicture}
\label{fig:POSMDP:example:bus:optimalpolicy:stop1}
}
\subfigure[Stop 2]{
\begin{tikzpicture}[>=stealth]
\begin{axis}[width=0.65*\textwidth,
                axis lines=middle,
                axis equal,
                xlabel=low,
                ylabel=medium,
                zlabel=high,
                xmin=0,
                xmax=1.25,
                ymin=0,
                ymax=1.25,
                zmin=0,
                zmax=2,
                xtick={0,1},
                ytick={0,1},
                ztick={0,1},
                view={135}{30}, thick]
\addplot3[patch, color=\twocolour!50, fill opacity=0.75, faceted color=black, very thick] coordinates{(1,0,0) (0,1,0) (0,0,1)};
\addplot3[mark=*] coordinates {(1,0,0)} node (s1') {};
\addplot3[mark=*] coordinates {(0,1,0)} node (s2') {};
\addplot3[mark=*] coordinates {(0,0,1)} node (s3') {};
\addplot3[mark=*, color = \twocolour, mark size = 1pt, only marks] table [x=i1, y=i2, z=i3, col sep=comma] {PolicyStop2Action0.csv};
\addplot3[mark=*, color = \threecolour, mark size = 1pt, only marks] table [x=i1, y=i2, z=i3, col sep=comma] {PolicyStop2Action1.csv};
\node[above left, xshift=-2ex, yshift=2ex] (s1) at (1,0,0) {\footnotesize $\begin{bmatrix}1\\0\\0\end{bmatrix}$};
\node[above right, xshift=2ex, yshift=2ex] (s2) at (0,1,0) {\footnotesize $\begin{bmatrix}0\\1\\0\end{bmatrix}$};
\node[right, xshift=3ex] (s3) at (0,0,1) {\footnotesize $\begin{bmatrix}0\\0\\1\end{bmatrix}$};
\path[->] (s1) edge[bend left] (s1');
\path[->] (s2) edge[bend right] (s2');
\path[->] (s3) edge[bend right] (s3');
\end{axis}
\end{tikzpicture}
\label{fig:POSMDP:example:bus:optimalpolicy:stop2}
} \hfill
\subfigure[Stop 3]{
\begin{tikzpicture}[>=stealth]
\begin{axis}[width=0.65*\textwidth,
                axis lines=middle,
                axis equal,
                xlabel=low,
                ylabel=medium,
                zlabel=high,
                xmin=0,
                xmax=1.25,
                ymin=0,
                ymax=1.25,
                zmin=0,
                zmax=2,
                xtick={0,1},
                ytick={0,1},
                ztick={0,1},
                view={135}{30}, thick]
\addplot3[patch, color=\twocolour!50, fill opacity=0.75, faceted color=black, very thick] coordinates{(1,0,0) (0,1,0) (0,0,1)};
\addplot3[mark=*] coordinates {(1,0,0)} node (s1') {};
\addplot3[mark=*] coordinates {(0,1,0)} node (s2') {};
\addplot3[mark=*] coordinates {(0,0,1)} node (s3') {};
\addplot3[mark=*, color = \twocolour, mark size = 1pt, only marks] table [x=i1, y=i2, z=i3, col sep=comma] {PolicyStop3Action0.csv};
\addplot3[mark=*, color = \threecolour, mark size = 1pt, only marks] table [x=i1, y=i2, z=i3, col sep=comma] {PolicyStop3Action1.csv};
\node[above left, xshift=-2ex, yshift=2ex] (s1) at (1,0,0) {\footnotesize $\begin{bmatrix}1\\0\\0\end{bmatrix}$};
\node[above right, xshift=2ex, yshift=2ex] (s2) at (0,1,0) {\footnotesize $\begin{bmatrix}0\\1\\0\end{bmatrix}$};
\node[right, xshift=3ex] (s3) at (0,0,1) {\footnotesize $\begin{bmatrix}0\\0\\1\end{bmatrix}$};
\path[->] (s1) edge[bend left] (s1');
\path[->] (s2) edge[bend right] (s2');
\path[->] (s3) edge[bend right] (s3');
\end{axis}
\end{tikzpicture}
\label{fig:POSMDP:example:bus:optimalpolicy:stop3}
}
\caption{The optimal policy for the bus problem at each bus stop with a regular mesh of belief states in which intensity the agent believes. A red dot \textcolor{\twocolour}{$\bullet$} represents a belief whose optimal action is to continue riding the bus, and a blue dot \textcolor{\threecolour}{$\bullet$} represents a belief whose optimal action is to stop riding the bus and take the bicycle.}
\label{fig:POSMDP:example:bus:optimalpolicy}
\end{figure}
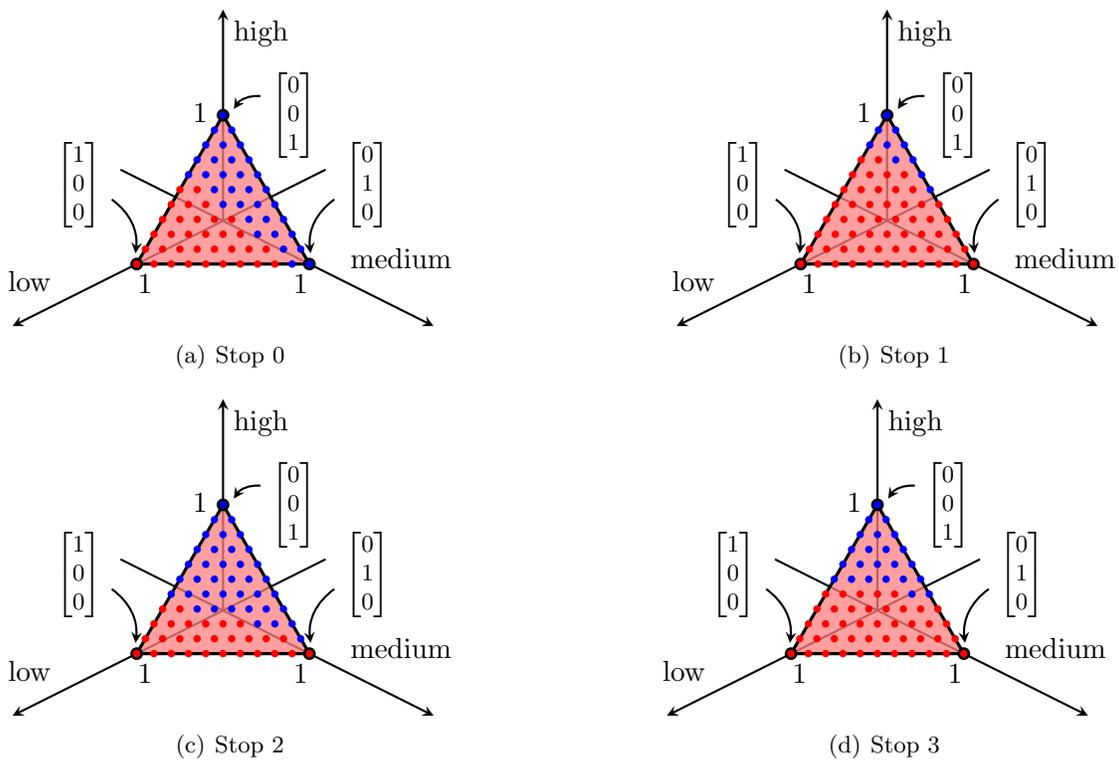

In this problem, we showed the ability of \textsc{ChronosPerseus} to quickly solve POMDPs that involve time without extending the problem state space to represent elapsed time.  This would have required the state space with an elapsed time dimension of at least $5\times 3 \times T$ where $T \geq 125$ for 1 state per time unit.  Instead, we naturally place the various durations into sojourn time distributions, and allowed the belief to be updated directly based on the likelihood of the distribution given the observed transitions.  Using \textsc{ChronosPerseus}, we see that POSMDPs are not more complicated to solve than their POMDP counterparts.  Instead, like using SMDPs in options, it allows for the compression of the temporal properties of the problem.  This simplifies not only the creation of the model, but also its evaluation by avoiding the explosion of the state space.  This approach also allows us to easily mix various temporal distributions and solve episodic problems as efficiently as non-episodic problems.

\subsection{A Machine Maintenance Problem}

\citet{Zhang2017} studied an industrial application of POSMDPs to maintain water (rapid gravity) filters. Filters are essential components in municipal drinking water treatment plants, where the filters can be classified into four states
\begin{equation*}
\mathcal{S} = \{ 1 = \text{good}, 2 = \text{acceptable}, 3 = \text{poor}, 4 = \text{awful}\}.
\end{equation*}
Unlike the bus problem in Section~\ref{sec:TheBusProblem}, the state of the filter is (fully) hidden and it must be inferred with uncertain observations. To extend the longevity and reliability of the filters, there are four maintenance actions, which are
\begin{equation*}
\mathcal{A} = \{ 1 = \text{do nothing}, 2 = \text{backwash}, 3 = \text{dose chemicals}, 4 = \text{replace}\}.
\end{equation*}
An observation to infer the state of the filter is to measure turbidity. Turbidity is the amount of cloudiness or haziness of a fluid caused by suspended particles: this is measured by a nephelometer, which shines a light beam through the water sample and measures how much light is reflected into a detector (often located at $90^\circ$ from the source beam). The turbidity of the incoming water into the filter and the turbidity of the outgoing water from the filter are measured. If the outgoing to incoming water turbidity ratio is close to zero, then the filter is in a good state. However, if the ratio is closer to one, the filter is likely to be in a poor or awful state. Hence, the observation space is given by the continuous interval $\mathcal{O} = [0,1]$.

The probability transition matrices are
\begin{align*}
P(\textcolor{\threecolour}{\cdot} \mid \textcolor{\twocolour}{\cdot}, 1) = P(\textcolor{\threecolour}{\cdot} \mid \textcolor{\twocolour}{\cdot}, 2) &= \begin{bNiceMatrix}[first-col, first-row, code-for-first-col={\scriptstyle\color{\twocolour}}, code-for-first-row={\scriptstyle\color{\threecolour}}] & s' = 1 & s' = 2 & s' = 3 & s'= 4\\
s=1 & 0.1043 & 0.7413 & 0.1493 & 0.0051\\
s=2 & 0      & 0.1043 & 0.7413 & 0.1544\\
s=3 & 0      & 0      & 0.1043 & 0.8957\\
s=4 & 0      & 0      & 0      & 1
\end{bNiceMatrix},\\ %
P(\textcolor{\threecolour}{\cdot} \mid \textcolor{\twocolour}{\cdot}, 3) &= \begin{bNiceMatrix}[first-col, first-row, code-for-first-col={\scriptstyle\color{\twocolour}}, code-for-first-row={\scriptstyle\color{\threecolour}}] & s' = 1 & s' = 2 & s' = 3 & s'= 4\\
s=1 & 1      & 0      & 0      & 0     \\
s=2 & 0.50   & 0.50   & 0      & 0     \\
s=3 & 0.25   & 0.70   & 0.05   & 0     \\
s=4 & 0.20   & 0.55   & 0.20   & 0.05
\end{bNiceMatrix},
\intertext{and}
P(\textcolor{\threecolour}{\cdot} \mid \textcolor{\twocolour}{\cdot}, 4) &= \begin{bNiceMatrix}[first-col, first-row, code-for-first-col={\scriptstyle\color{\twocolour}}, code-for-first-row={\scriptstyle\color{\threecolour}}] & s' = 1 & s' = 2 & s' = 3 & s'= 4\\
s=1 & 1      & 0      & 0      & 0     \\
s=2 & 1      & 0      & 0      & 0     \\
s=3 & 1      & 0      & 0      & 0     \\
s=4 & 1      & 0      & 0      & 0   
\end{bNiceMatrix}.
\end{align*}

The sojourn times for actions 1, 2, and 3 are fixed, whereas $\tau(s' \mid s, 1) = 78.7433$, $\tau(s' \mid s, 2) = 85.3052$, and $\tau(s' \mid s, 3) = 3$. The sojourn time for action 4 follows a truncated Gaussian distribution with mean $\mu = 10$, standard deviation $\sigma = 1.5$, and $\tau(s' \mid s, 4) > 0$; in other words,
\begin{equation}
f(\tau \mid s, a = 4, s') = \begin{cases}
 \dfrac{1}{1.5}\cdot \dfrac{\varphi\left(\dfrac{\tau-10}{1.5}\right)}{1 - \Phi\left(-\dfrac{10}{1.5}\right)}, & \text{if } \tau > 0;\\
 0, & \text{otherwise,}
\end{cases}
\end{equation}
where
\begin{equation}
\varphi(x) = \frac{1}{\sqrt{2\pi}} \exp\left(-\frac12 x^2\right)
\end{equation}
is the standard normal probability density function, and $\Phi$ is its corresponding cumulative distribution function.

To deal with continuous observations, we discretize the interval $[0,1]$ into $j$ evenly spaced numbers, which yields a finite set of observations $O$ ($|O| = j$). The observation transition probability is 
\begin{equation*}
G(o \mid a, s') = G(o \mid a) = \frac{f_B\big(o \mid \varphi(a), \eta(a)\big)}{\displaystyle\sum_{o' \in O} f_B\big(o' \mid \varphi(a), \eta(a)\big)}
\end{equation*}
where 
\begin{equation*}
    f_B\big(o \mid \varphi(a), \eta(a)\big) = \frac{\Gamma(\varphi + \eta)}{\Gamma(\varphi)\Gamma(\eta)} o^{\varphi - 1}(1 - o)^{\eta - 1}, \qquad \Gamma(z) = \int_0^\infty x^{z-1}{\e}^{-x}\, \d x, \ \Re(z) > 0
\end{equation*}
is the beta probability density function and the parameters 
\begin{equation*}
\varphi(\textcolor{\twocolour}{\cdot}) = \begin{bNiceMatrix}[first-col, code-for-first-col={\scriptstyle\color{\twocolour}}] a = 1 & 2 \\ a = 2 & 6 \\ a = 3 & 18 \\ a = 4 & 18 \\  \end{bNiceMatrix} \qquad \text{and} \qquad \eta(\textcolor{\twocolour}{\cdot}) = \begin{bNiceMatrix}[first-col, code-for-first-col={\scriptstyle\color{\twocolour}}] a = 1 & 18 \\ a = 2 & 18 \\ a = 3 & 18 \\ a = 4 & 6 \\  \end{bNiceMatrix}.
\end{equation*}
The beta probability density functions are shown in Fig.~\ref{fig:POSMDP:ex:Zhang:betapdfs}.
\begin{figure}[h!t]
\centering
\input{Fig/zhangBetadistributions.tex}
\caption{The beta probability density functions used for the observation transition probabilities in the machine maintenance problem.}
\label{fig:POSMDP:ex:Zhang:betapdfs}
\end{figure}
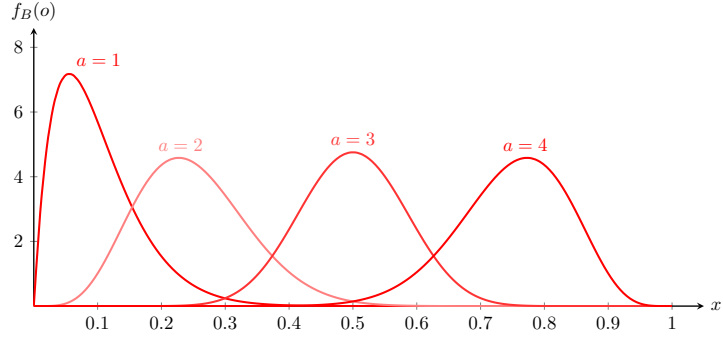

The lump sum reward is 
\begin{equation*}
r_1(s, \textcolor{\twocolour}{\cdot}) = \begin{bNiceMatrix}[first-col, code-for-first-col={\scriptstyle\color{\twocolour}}] a = 1 & 0 \\ a = 2 & -100 \\ a = 3 & -200 \\ a = 4 & -500 \\  \end{bNiceMatrix} \qquad \forall s \in \mathcal{S},
\end{equation*}
and the continuous reward rate is
\begin{equation*}
r_2(\textcolor{\twocolour}{\cdot}, \textcolor{\threecolour}{\cdot}, s') = \begin{bNiceMatrix}[first-col, first-row, code-for-first-col={\scriptstyle\color{\twocolour}}, code-for-first-row={\scriptstyle\color{\threecolour}}] & a = 1 & a = 2 & a = 3 & a = 4\\
s=1 & 500    & 500    & -100   & -100  \\
s=2 & 250    & 250    & -100   & -100  \\
s=3 & -300   & -300   & -100   & -100  \\
s=4 & -500   & -500   & -100   & -100
\end{bNiceMatrix} \qquad \forall s' \in \mathcal{S}.
\end{equation*}

The initial $\alpha$-vector is
\begin{equation}
\alpha(\textcolor{\twocolour}{\cdot}) = \begin{bNiceMatrix}[first-col, code-for-first-col={\scriptstyle\color{\twocolour}}] s = 1 & -10^6 \\ s = 2 & -10^6 \\ s = 3 & -10^6 \\ s = 4 & -10^6 \\  \end{bNiceMatrix}.
\end{equation}

The initial belief $\xi_0$ is that the filter is good; that is, 
\begin{equation}
\xi_0(\textcolor{\twocolour}{\cdot}) = \begin{bNiceMatrix}[first-col, code-for-first-col={\scriptstyle\color{\twocolour}}] s = 1 & 1 \\ s = 2 & 0 \\ s = 3 & 0 \\ s = 4 & 0\\  \end{bNiceMatrix}.
\end{equation}

\noindent \textbf{Results.} Running \textsc{ChronosPerseus} on the problem with discount rate $\beta = 0.01$ generates the set
{\small
\begin{align*}
V &= \left\{\begin{bNiceMatrix}[code-for-first-row={\scriptstyle \color{\threecolour}}, first-row, code-for-first-col={\scriptstyle \color{\twocolour}},  first-col] & a = \text{backwash}\\ V(\text{good}) & 46351.3242\\ V(\text{acceptable}) & 31551.5723 \\ V(\text{poor}) & -71.0176 \\ V(\text{awful}) & -11550.9883 \end{bNiceMatrix}, \begin{bNiceMatrix}[code-for-first-row={\scriptstyle \color{\threecolour}}, first-row] a = \text{dose} \\ 43257.2109\\ 41467.2891\\ 40529.1953\\ 40172.0430\end{bNiceMatrix}, \begin{bNiceMatrix}[code-for-first-row={\scriptstyle \color{\threecolour}}, first-row] a = \text{dose} \\ 43404.0859\\ 41505.2148\\ 40516.1875\\ 40159.6211\end{bNiceMatrix}, \begin{bNiceMatrix}[code-for-first-row={\scriptstyle \color{\threecolour}}, first-row] a = \text{dose} \\ 43690.0625\\ 41519.4531\\ 40405.4766\\ 40057.0312 \end{bNiceMatrix}, 
\begin{bNiceMatrix}[code-for-first-row={\scriptstyle \color{\threecolour}}, first-row] a = \text{dose} \\ 43987.1328\\ 41352.2109\\ 40036.6328\\ 39764.9141 \end{bNiceMatrix},\right.\\
&\qquad \begin{bNiceMatrix}[code-for-first-row={\scriptstyle \color{\threecolour}}, first-row, code-for-first-col={\scriptstyle \color{\twocolour}},  first-col] & a = \text{dose}\\ V(\text{good}) & 43567.4844\\ V(\text{acceptable}) & 41526.0000 \\ V(\text{poor}) & 40470.0547 \\ V(\text{awful}) & 40112.9766 \end{bNiceMatrix}, \begin{bNiceMatrix}[code-for-first-row={\scriptstyle \color{\threecolour}}, first-row] a = \text{dose} \\ 44235.8906\\ 40804.3320\\ 39155.7461\\ 39067.2930\end{bNiceMatrix}, \begin{bNiceMatrix}[code-for-first-row={\scriptstyle \color{\threecolour}}, first-row] a = \text{dose} \\ 43890.7227\\ 41438.7500\\ 40201.9414\\ 39896.9844\end{bNiceMatrix}, \begin{bNiceMatrix}[code-for-first-row={\scriptstyle \color{\threecolour}}, first-row] a = \text{dose} \\ 43800.3984\\ 41489.0391\\ 40313.0742\\ 39984.2852 \end{bNiceMatrix}, 
\begin{bNiceMatrix}[code-for-first-row={\scriptstyle \color{\threecolour}}, first-row] a = \text{dose} \\ 44187.7969\\ 40967.9141\\ 39406.7695\\  39265.6484 \end{bNiceMatrix},\\
&\qquad \qquad \qquad \left. \begin{bNiceMatrix}[code-for-first-row={\scriptstyle \color{\threecolour}}, first-row, code-for-first-col={\scriptstyle \color{\twocolour}}, first-col] & a = \text{dose} \\ V(\text{good}) & 44127.5703\\ V(\text{acceptable}) & 41123.7617\\ V(\text{poor}) & 39652.8359\\ V(\text{awful}) & 39461.3516\end{bNiceMatrix}, \begin{bNiceMatrix}[code-for-first-row={\scriptstyle \color{\threecolour}}, first-row] a = \text{dose} \\ 44061.6992\\  41249.8320\\  39859.4453\\ 39625.1172\end{bNiceMatrix}, \begin{bNiceMatrix}[code-for-first-row={\scriptstyle \color{\threecolour}}, first-row] a = \text{replace} \\ 40504.4414\\ 40504.4414\\ 40504.4414\\ 40504.4414\end{bNiceMatrix}\right\},
\end{align*}}
where $V$ contains 13 $\alpha$-vectors with their corresponding actions labeled at the top.
With $|B| = 5000$ sampled belief points and 40 iterations, it took around 16 hours for \citet{Zhang2017} (Intel Core i5-4590 CPU @ 3.30 GHz) to complete their calculations, while \textsc{ChronosPerseus} took less than 40 seconds (Intel Core Xeon Silver 4210 CPU @ 2.20 Ghz, Nvidia GeForce RTX 2080 Super 8 GB). Their paper did not state the number of bins of observations, so we selected 100 for our simulation. 

Table~\ref{tab:POSMDP:ex:Zhang:beliefwithoptimalaction} lists some belief points with their corresponding optimal value and optimal action from \textsc{ChronosPerseus} alongside the results of \citet{Zhang2017}. The beliefs in Table~\ref{tab:POSMDP:ex:Zhang:beliefwithoptimalaction} were selected by \citeauthor{Zhang2017}, and these values were included in the sampled belief set $B$ of \textsc{ChronosPerseus} as well. At the top of the column, these beliefs correspond to a filter believed to be in good condition and descend to acceptable, poor, and finally awful. We notice that the \textsc{ChronosPerseus} has selected to backwash for a filter believed to be in good condition compared to \citeauthor{Zhang2017} whose action is do-nothing. The accumulated reward for backwash is more than the accumulated reward for doing nothing: $R(1,1) = 27249.43$ while $R(1,2) = 28594.38$. There is more time during the backwash, which allows accumulation of continuous reward overcoming the initial lump negative reward of $-100$ from backwash compared to the time allowed for the do-nothing action. The accumulated reward has to compensate for the risk of getting bad water (such as ending up in the poor or awful state). Hence, we believe that \textsc{ChronosPerseus} has selected the correct action.

\begin{table}[h!t]
\caption{Sample of belief points, their corresponding optimal values and actions}\label{tab:POSMDP:ex:Zhang:beliefwithoptimalaction}
{\footnotesize
    \begin{tabular}{@{}lcccc@{}}
    \toprule
		             & \multicolumn{2}{c}{\textsc{ChronosPerseus}} & \multicolumn{2}{c}{\citet{Zhang2017}} \\
								\cline{2-3} \cline{4-5}
    Belief state, $\xi$ & Optimal Value & Optimal Action & Optimal Value & Optimal Action\\
    \midrule
    $\begin{bmatrix}0.9972 & 0.0028 & 0.0000 & 0.0000 \end{bmatrix}^\top$ & 46309.8867 & 2 & 46316.40 & 1 \\
    $\begin{bmatrix}0.9965 & 0.0035 & 0.0000 & 0.0000 \end{bmatrix}^\top$ & 46299.5234 & 2 & 46306.04 & 1 \\
    $\begin{bmatrix}0.8714 & 0.1286 & 0.0000 & 0.0000 \end{bmatrix}^\top$ & 44448.0742 & 2 & 44454.43 & 2 \\
    $\begin{bmatrix}0.8160 & 0.1840 & 0.0000 & 0.0000 \end{bmatrix}^\top$ & 43628.1680 & 2 & 43634.51 & 2 \\
    $\begin{bmatrix}0.0031 & 0.6803 & 0.3165 & 0.0001 \end{bmatrix}^\top$ & 41197.9805 & 3 & 41215.31 & 3 \\
    $\begin{bmatrix}0.0001 & 0.0390 & 0.9457 & 0.0152 \end{bmatrix}^\top$ & 40560.6250 & 3 & 40574.81 & 3 \\
    $\begin{bmatrix}0.0000 & 0.0003 & 0.8488 & 0.1509 \end{bmatrix}^\top$ & 40504.4453 & 4 & 40498.43 & 4 \\
    $\begin{bmatrix}0.0000 & 0.0000 & 0.0000 & 1.0000 \end{bmatrix}^\top$ & 40504.4414 & 4 & 40385.84 & 4 \\
    \bottomrule
    \end{tabular}%
}

\end{table}

In this section, we showed that \textsc{ChronosPerseus} could solve real-world POSMDP problems with continuous observation spaces and that importance sampling combined with GPU implementation provides few orders of magnitude speed up the computations allowing for a higher number of iterations and thus a longer horizon. 

\section{Summary and Conclusion}
The partially observable semi-Markov decision process (POSMDP) model has been used for decades in the operations research community to provide planning under uncertainty. While it has been used for tasks such as maintenance scheduling, it can be equally valuable in the artificial intelligence community by modeling not only the interaction between an agent and its environment but also dealing with the issue of time---that is, how long it takes for the agent to transition from one hidden state to another. Moreover, solving a POSMDP rather than an equivalent time-extended state-space POMDP using importance sampling is faster by avoiding the huge state-space explosion that adding time-steps would require, in a way similar to the use of SMDPs and options in reinforcement learning.

In this paper, we presented a new POSMDP solver---\textsc{ChronosPerseus}---combining \textsc{Perseus} and importance sampling to solve a POSMDP where the transition time is observable. We showed that the solver works on episodic and non-episodic problems, with mixed-observably, discrete or continuous observation space, and a mixture of fixed and stochastic continuous sojourn times. Furthermore, we demonstrate how it could learn a policy on a POMDP where the time is the only available information to resolve the partially observable state, as this may happen in many real-world problems. 


\acks{This research has been supported by CD-ARP research grants from Canadian Defence Academy and the Royal Military College of Canada to Dr.~F.~Rivest.   Icons in Figure~\ref{fig:example:BusProblem} are provided by \href{https://thenounproject.com}{The Noun Project}: \textit{Bus} by Rainbow Designs, \textit{Bus Stop Square} by Rihards Gromuls, \textit{man and bicycle} by Gan Khoon Lay, and \textit{clock} and \textit{Time} by kiddo.}

\vskip 0.2in
\bibliography{main.bib}

\end{document}

%% file: Fig/POSMDPdiagram2.tex
\usetikzlibrary{arrows,automata}
\begin{tikzpicture}[->,>=stealth,shorten >=1pt,auto,node distance=1.5cm,
                    thick]

\node[initial,state] (S0)                    {$S_0$};
\node[state, fill=\twocolour!20] (O0) [below of = S0] {$O_0$};
\node[state, fill=\twocolour!20]         (T0) [above of=S0] {$T_0$};
\node[state, fill=\twocolour!50]         (A0) [right of=S0] {$A_0$};
\node[state]         (S1) [right of=A0] {$S_1$};
\node[state, fill=\twocolour!20]         (T1) [above of=S1] {$T_1$};
\node[state, fill=\twocolour!20]         (O1) [below of=S1] {$O_1$};

\node[state, fill=\twocolour!50]         (A1) [right of=S1] {$A_1$};
\node[state]         (S2) [right of=A1] {$S_2$};
\node[state, fill=\twocolour!20]         (T2) [above of=S2] {$T_2$};
\node[state, fill=\twocolour!20]         (O2) [below of=S2] {$O_2$};

\node[] [right of =S2] (A2) {$\cdots$};

\path (S0) edge (T0) edge (O0)
          (T0) edge (A0)
          (O0) edge (A0)
          (S1) edge (T1) edge (O1)
          (A0) edge (T1) edge (O1)
          (A0) edge (S1)
          (A1) edge (T2) edge (O2)
          (A1) edge (S2)
          (S2) edge (T2) edge (O2)
          (T1) edge (A1)
          (O1) edge (A1)
          (T2) edge (A2)
          (O2) edge (A2);

\path[dashed] (S0) edge  (A0)
                        (S1) edge (A1)
			    (S2) edge (A2);
			    
\draw[decorate, decoration={brace, amplitude=5pt}, -, \twocolour!30] (0,2.1) -- (3,2.1) node[midway, yshift=3pt, black]{$\tau_0$} ;
\draw[decorate, decoration={brace, amplitude=5pt}, -, \twocolour!30] (3,2.1) -- (6,2.1) node[midway, yshift=3pt, black]{$\tau_1$};

\end{tikzpicture}

%% file: Fig/POMDPvsPOSMDPtime.tex
\begin{tikzpicture}[>=stealth,thick]
\node[left] at (-2.5,2.2) {POMDP:};
\draw (-2,2.2) -- (2.4,2.2);

\draw (-2,2.1) -- (-2,2.3)node[above]{$t_0$};
\draw (-0.7,2.1) -- (-0.7,2.3)node[above]{$t_1$};
\draw[->, \twocolour] (-1.8,2.7) to[out=30, in=150] (-0.9,2.7) {};

\draw (0.6,2.1) -- (0.6,2.3)node[above]{$t_2$};
\draw[->, \twocolour] (-0.5,2.7) to[out=30, in=150] (0.4,2.7) {};

\draw (1.9,2.1) -- (1.9,2.3)node[above]{$t_3$};
\draw[->, \twocolour] (0.8,2.7) to[out=30, in=150] (1.7,2.7) {};

\draw (3.62,2.1) -- (3.62,2.3)node[above]{$t_n$};
\draw[->, \twocolour] (2.1,2.7) to[out=30, in=150] (3.42,2.7) {};
\draw[white, fill=white]  (2.4,3) rectangle (3.1,2.7);

\draw (4.92,2.1) -- (4.92,2.3)node[above]{$t_{n+1}$};
\draw[->, \twocolour] (3.82,2.7) to[out=30, in=150] (4.52,2.7) {};

\draw [decorate, decoration={brace, amplitude=5pt}] (4.92,2) --node[midway,below,yshift=-1ex, \twocolour]{$\tau = 1$} (3.62,2);
\draw[->] (3.1,2.2) -- (5.62,2.2)node[below right]{time};
\node at (2.8,2.2) {$\cdots$};

\node[left] at (-2.5,0.5) {POSMDP:};

\draw (-2,0.5) -- (2.4,0.5);

\draw (-2,0.4) -- (-2,0.6)node[above]{$t_0$};
\draw (-1.2,0.4) -- (-1.2,0.6)node[above]{$t_1$};
\draw[->, \twocolour] (-1.8,1) to[out=30, in=150] (-1.4,1) {};

\draw (1.1,0.4) -- (1.1,0.6)node[above]{$t_2$};
\draw[->, \twocolour] (-1,1) to[out=30, in=150] (0.9,1) {};

\draw (1.9,0.4) -- (1.9,0.6)node[above]{$t_3$};
\draw[->, \twocolour] (1.3,1) to[out=30, in=150] (1.7,1) {};

\draw (3.62,0.4) -- (3.62,0.6)node[above]{$t_n$};
\draw[->, \twocolour] (2.1,1) to[out=30, in=150] (3.42,1) {};
\draw[white, fill=white]  (2.4,1.3) rectangle (3.1,1);

\draw (5.32,0.4) -- (5.32,0.6)node[above]{$t_{n+1}$};
\draw[->, \twocolour] (3.82,1) to[out=30, in=150] (4.92,1) {};

\draw [decorate, decoration={brace, amplitude=5pt}] (5.32,0.3) --node[midway,below,yshift=-1ex, \twocolour]{$\tau \sim F(\tau \mid s, a, s')$} (3.62,0.3);
\draw[->] (3.1,0.5) -- (5.62,0.5)node[below right]{time};
\node at (2.8,0.5) {$\cdots$};

\end{tikzpicture}

%% file: Fig/POMDPassumptionAlpha1.tex
\begin{tikzpicture}[scale=0.3, >=stealth,thick]
\usetikzlibrary{arrows}

\draw[black!0] (2.5,3)node[below,black]{$1$} -- (2.5,5.5);
\draw[<->] (-1.5,5.75)node[left]{$V$} -- (-1.5,3)node[below]{$0$} -- (3,3)node[right]{$\xi(s_1)$};

\draw[black!0] (2.5,-1.5)node[below,black]{$1$} -- (2.5,1);
\draw[<->] (-1.5,1.25)node[left]{$V$} -- (-1.5,-1.5)node[below]{$0$} -- (3,-1.5)node[right]{$\xi(s_2)$};

\end{tikzpicture}

%% file: Fig/POMDPassumptionAlpha2.tex
\begin{tikzpicture}[scale=0.3, >=stealth,thick]
\usetikzlibrary{arrows}
\draw[black!50] (2.5,3)node[below,black]{$1$} -- (2.5,5.75);
\draw[<->] (-1.5,6)node[left]{$V$} -- (-1.5,3)node[below]{$0$} -- (3,3)node[right]{$\xi(s_1)$};

\draw[fill=black] (-1.5,4) circle (0.1) node[left]{\footnotesize $V(s_2)$};
\draw[fill=black] (2.5,4.75) circle (0.1) node[right]{\footnotesize $V(s_1)$};

\draw[black!50] (2.5,-1.5)node[below,black]{$1$} -- (2.5,1.25);
\draw[<->] (-1.5,1.5)node[left]{$V$} -- (-1.5,-1.5)node[below]{$0$} -- (3,-1.5)node[right]{$\xi(s_2)$};

\draw[fill=black] (-1.5,0.25) circle (0.1) node[left]{\footnotesize $V(s_1)$};
\draw[fill=black] (2.5,-0.5) circle (0.1) node[right]{\footnotesize $V(s_2)$};

\end{tikzpicture}

%% file: Fig/POMDPassumptionAlpha3.tex
\begin{tikzpicture}[scale=0.3, >=stealth,thick]
\usetikzlibrary{arrows}
\draw[black!50] (2.5,3)node[below,black]{$1$} -- (2.5,5.75);
\draw[<->] (-1.5,6)node[left]{$V$} -- (-1.5,3)node[below]{$0$} -- (3,3)node[right]{$\xi(s_1)$};

\draw[very thick, \twocolour]  (-1.5,4) -- (2.5,4.75); 

\draw[fill=black] (-1.5,4) circle (0.1) node[left]{\footnotesize $V(s_2)$};
\draw[fill=black] (2.5,4.75) circle (0.1) node[right]{\footnotesize $V(s_1)$};

\draw[black!50] (2.5,-1.5)node[below,black]{$1$} -- (2.5,1.25);
\draw[<->] (-1.5,1.5)node[left]{$V$} -- (-1.5,-1.5)node[below]{$0$} -- (3,-1.5)node[right]{$\xi(s_2)$};

\draw[very thick, \twocolour]  (-1.5,0.25) -- (2.5,-0.5);

\draw[fill=black] (-1.5,0.25) circle (0.1) node[left]{\footnotesize $V(s_1)$};
\draw[fill=black] (2.5,-0.5) circle (0.1) node[right]{\footnotesize $V(s_2)$};

\end{tikzpicture}

%% file: Fig/busProblem/busProblem.tex
\begin{tikzpicture}
\node at (-1.25,0.25) {\includegraphics[scale=0.07]{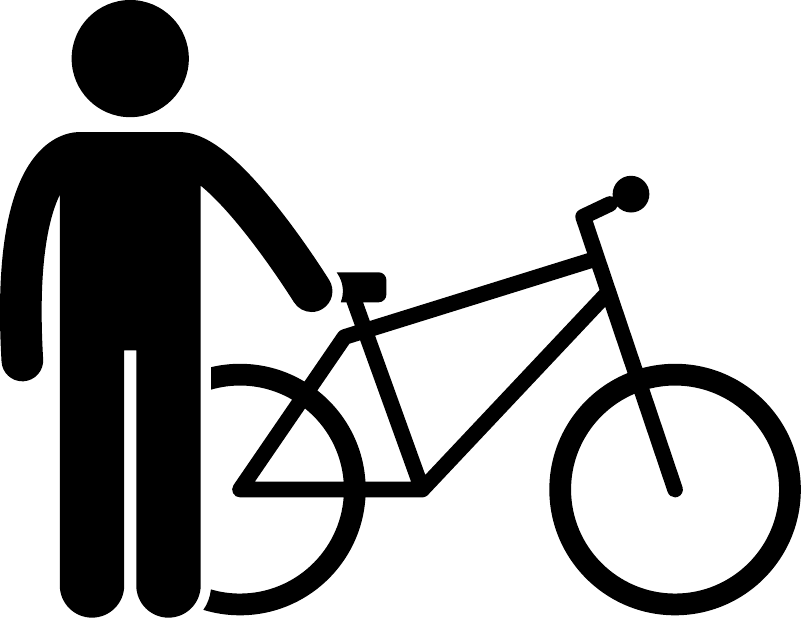}};
\node at (-3,0.35) {\includegraphics[scale=0.75]{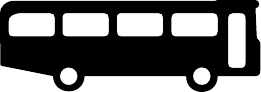}};

\node at (-1.05,0.4) {\includegraphics[scale=0.35]{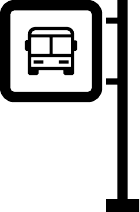}};
\node[below] at (-0.85,0) {\footnotesize \textcolor{\threecolour}{0}};

\node at (0,0.4) {\includegraphics[scale=0.35]{Fig/busProblem/busStop.pdf}};
\node[below] at (0.2,0) {\footnotesize \textcolor{\threecolour}{1}};

\draw[decorate, decoration={brace, amplitude=5pt}] (0.2,-0.4)-- (-0.85,-0.4);

\node at (1,0.4) {\includegraphics[scale=0.35]{Fig/busProblem/busStop.pdf}};
\node[below] at (1.2,0) {\footnotesize \textcolor{\threecolour}{2}};

\draw[decorate, decoration={brace, amplitude=5pt}] (1.2,-0.4)-- (0.2,-0.4);

\node at (2,0.4) {\includegraphics[scale=0.35]{Fig/busProblem/busStop.pdf}};
\node[below] at (2.2,0) {\footnotesize\textcolor{\threecolour}{3}};

\draw[decorate, decoration={brace, amplitude=5pt}] (2.2,-0.4)-- (1.2,-0.4);

\node at (3,0.4) {\includegraphics[scale=0.35]{Fig/busProblem/busStop.pdf}};
\node[below] at (3.2,0) {\footnotesize\textcolor{\threecolour}{4}};

\draw[decorate, decoration={brace, amplitude=5pt}] (3.2,-0.4)-- (2.2,-0.4);

\draw (-3.75,0) -- (3.5,0);

\node at (-0.325,-0.85) {\includegraphics[scale=0.15]{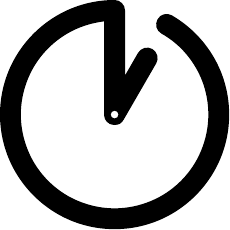}};
\node[below] at (-0.325,-0.95) {\footnotesize ?};
\node at (0.7,-0.85) {\includegraphics[scale=0.15]{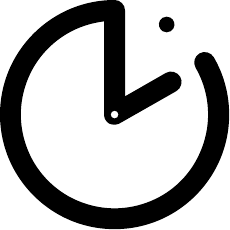}};
\node[below] at (0.7,-0.95) {\footnotesize ?};
\node at (1.675,-0.85) {\includegraphics[scale=0.15]{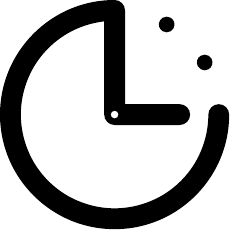}};
\node[below] at (1.675,-0.95) {\footnotesize ?};
\node at (2.675,-0.85) {\includegraphics[scale=0.15]{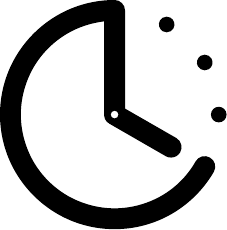}};
\node[below] at (2.675,-0.95) {\footnotesize ?};
\end{tikzpicture}

%% file: Fig/busProblemBus.tex
\begin{tikzpicture}[->,>=stealth,shorten >=1pt,auto, thick,node distance=1.3cm]
\tikzstyle{every state}=[fill=\twocolour,draw=none,text=white]

\draw[fill=black!20, black!20] (-0.5,-0.75) rectangle (0.5,-7);
\node[below] at (0,-7) {\footnotesize Low intensity};

\draw[fill=black!20, black!20, rotate=135] (-0.5,-0.75) rectangle (0.5,-7);
\node[above, rotate=-45] at (4.95,4.925) {\footnotesize Medium intensity};

\draw[fill=black!20, black!20, rotate=-135] (-0.5,-0.75) rectangle (0.5,-7);
\node[above, rotate=45] at (-4.925,4.925) {\footnotesize High intensity};

\node[state] (S){\footnotesize Start};

\node[state] (S01) [below of = S]{\scriptsize $(0,1)$};
\node[state] (S11) [below of=S01]{\scriptsize $(1,1)$};
\path (S01) edge node[left]{\footnotesize $1$} (S11);
\node[state] (S21) [below of=S11]{\scriptsize $(2,1)$};
\path (S11) edge node[left]{\footnotesize $1$} (S21);
\node[state] (S31) [below of=S21]{\scriptsize $(3,1)$};
\path (S21) edge node[left]{\footnotesize $1$} (S31);
\node[state] (S41) [below of=S31]{\scriptsize $(4,1)$};
\path (S31) edge node[left]{\footnotesize $1$} (S41);

\node[state] (S02)[above right of=S]{\scriptsize $(0,2)$};
\node[state] (S12)[above right of=S02]{\scriptsize $(1,2)$};
\path (S02) edge node{\footnotesize $1$} (S12);
\node[state] (S22)[above right of=S12]{\scriptsize $(2,2)$};
\path (S12) edge node{\footnotesize $1$} (S22);
\node[state] (S32)[above right of=S22]{\scriptsize $(3,2)$};
\path (S22) edge node{\footnotesize $1$} (S32);
\node[state] (S42)[above right of=S32]{\scriptsize $(4,2)$};
\path (S32) edge node{\footnotesize $1$} (S42);

\node[state] (S03)[above left of=S]{\scriptsize $(0,3)$};
\node[state] (S13)[above left of=S03]{\scriptsize $(1,3)$};
\path (S03) edge[above right] node{\footnotesize $1$} (S13);
\node[state] (S23)[above left of=S13]{\scriptsize $(2,3)$};
\path (S13) edge[above right] node{\footnotesize $1$} (S23);
\node[state] (S33)[above left of=S23]{\scriptsize $(3,3)$};
\path (S23) edge[above right] node{\footnotesize $1$} (S33);
\node[state] (S43)[above left of=S33]{\scriptsize $(4,3)$};
\path (S33) edge[above right] node{\footnotesize $1$} (S43);

\path (S) edge node[left]{\footnotesize $\frac13$} (S01);
\path (S) edge node[]{\footnotesize $\frac13$} (S02);
\path (S) edge node[above right]{\footnotesize $\frac13$} (S03);

\path (S41) edge [bend right] node[right]{\footnotesize $\frac13$} (S01);
\path (S41) edge [bend right] node[right]{\footnotesize $\frac13$} (S02);
\path (S41) edge [bend left] node[left]{\footnotesize $\frac13$} (S03);
\path (S42) edge [bend left] node[left, xshift=-1ex]{\footnotesize $\frac13$} (S02);
\path (S42) edge [bend left] node{\footnotesize $\frac13$} (S01);
\path (S42) edge [bend right] node{\footnotesize $\frac13$} (S03);
\path (S43) edge [bend right=30] node[left, xshift=-1ex]{\footnotesize $\frac13$} (S03);
\path (S43) edge [bend left] node{\footnotesize $\frac13$} (S02);
\path (S43) edge [bend right] node[left]{\footnotesize $\frac13$} (S01);
\end{tikzpicture}

%% file: Fig/busProblemBike.tex
\begin{tikzpicture}[->,>=stealth,shorten >=1pt,auto, thick,node distance=1.3cm]
\tikzstyle{every state}=[fill=\twocolour,draw=none,text=white]

\draw[fill=black!20, black!20] (-0.5,-0.75) rectangle (0.5,-7);
\node[below] at (0,-7) {\footnotesize Low intensity};

\draw[fill=black!20, black!20, rotate=135] (-0.5,-0.75) rectangle (0.5,-7);
\node[above, rotate=-45] at (4.95,4.925) {\footnotesize Medium intensity};

\draw[fill=black!20, black!20, rotate=-135] (-0.5,-0.75) rectangle (0.5,-7);
\node[above, rotate=45] at (-4.925,4.925) {\footnotesize High intensity};

\node[state] (S){\footnotesize Start};

\node[state] (S01) [below of = S]{\scriptsize $(0,1)$};
\node[state] (S11) [below of=S01]{\scriptsize $(1,1)$};
\node[state] (S21) [below of=S11]{\scriptsize $(2,1)$};
\node[state] (S31) [below of=S21]{\scriptsize $(3,1)$};
\node[state] (S41) [below of=S31]{\scriptsize $(4,1)$};

\node[state] (S02)[above right of=S]{\scriptsize $(0,2)$};
\node[state] (S12)[above right of=S02]{\scriptsize $(1,2)$};
\node[state] (S22)[above right of=S12]{\scriptsize $(2,2)$};
\node[state] (S32)[above right of=S22]{\scriptsize $(3,2)$};
\node[state] (S42)[above right of=S32]{\scriptsize $(4,2)$};

\node[state] (S03)[above left of=S]{\scriptsize $(0,3)$};
\node[state] (S13)[above left of=S03]{\scriptsize $(1,3)$};
\node[state] (S23)[above left of=S13]{\scriptsize $(2,3)$};
\node[state] (S33)[above left of=S23]{\scriptsize $(3,3)$};
\node[state] (S43)[above left of=S33]{\scriptsize $(4,3)$};

\path (S) edge node[left]{\footnotesize $\frac13$} (S01);
\path (S) edge node[]{\footnotesize $\frac13$} (S02);
\path (S) edge node[above right]{\footnotesize $\frac13$} (S03);

\path (S01) edge [bend right] node[xshift=1.2ex, yshift=13ex]{\footnotesize $1$} (S41);
\path (S11) edge [bend right] node[yshift=9ex]{\footnotesize $1$} (S41);
\path (S21) edge [bend right] node[xshift=-1ex, yshift=5ex]{\footnotesize $1$} (S41);
\path (S31) edge node[left,xshift=0.5ex]{\footnotesize $1$} (S41);

\path (S02) edge [bend right] node[xshift=-11ex,yshift=-10ex]{\footnotesize $1$} (S42);
\path (S12) edge [bend right] node[xshift=-6ex, yshift=-7ex]{\footnotesize $1$} (S42);
\path (S22) edge [bend right] node[xshift=-3ex, yshift=-5ex]{\footnotesize $1$} (S42);
\path (S32) edge node[right, xshift=-0.5ex, yshift=-1ex]{\footnotesize $1$} (S42);

\path (S03) edge [bend right] node[xshift=9ex, yshift=-10ex]{\footnotesize $1$} (S43);
\path (S13) edge [bend right] node[xshift=7ex, yshift=-6ex]{\footnotesize $1$} (S43);
\path (S23) edge [bend right] node[xshift=5ex, yshift=-2.5ex]{\footnotesize $1$} (S43);
\path (S33) edge node[right,yshift=0.5ex]{\footnotesize $1$} (S43);

\path (S41) edge [bend right] node[right]{\footnotesize $\frac13$} (S01);
\path (S41) edge [bend right] node[right]{\footnotesize $\frac13$} (S02);
\path (S41) edge [bend left] node[left]{\footnotesize $\frac13$} (S03);

\path (S42) edge [bend right] node[left, xshift=-1ex]{\footnotesize $\frac13$} (S02);
\path (S42) edge [bend left] node{\footnotesize $\frac13$} (S01);
\path (S42) edge [bend right] node[above]{\footnotesize $\frac13$} (S03);

\path (S43) edge [bend right=30] node[left, xshift=-1ex]{\footnotesize $\frac13$} (S03);
\path (S43) edge [bend left] node{\footnotesize $\frac13$} (S02);
\path (S43) edge [bend right] node[left]{\footnotesize $\frac13$} (S01);

\end{tikzpicture}

%% file: Fig/zhangBetadistributions.tex
\begin{tikzpicture}[
  scale = 0.65,
  declare function={
    gamma(\z) =
    (2.506628274631*sqrt(1/\z)+0.20888568*(1/\z)^(1.5)+
    0.00870357*(1/\z)^(2.5)-(174.2106599*(1/\z)^(3.5))/25920-
    (715.6423511*(1/\z)^(4.5))/1244160)*exp((-ln(1/\z)-1)*\z);
  },
  declare function={
    beta(\a,\b) = gamma(\a)*gamma(\b)/gamma(\a+\b);
  },
  declare function={
    betapdf(\x,\a,\b) = \x^(\a-1)*(1-\x)^(\b-1)/beta(\a,\b);
  }]

  \begin{axis}[every axis plot post/.append style={
  mark=none,
  samples=100,
  smooth, 
  >=stealth, 
  very thick,
  domain=0:1}, 
  axis lines = middle,
  xlabel=$x$,
  ylabel=$f_B(o)$,
  every axis x label/.style={
    at={(ticklabel* cs:1.0)},
    anchor=west,
},
every axis y label/.style={
    at={(ticklabel* cs:1.0)},
    anchor=south,
},
  thick,
  xmin = 0,
  xmax = 1.05,
  ymin = 0,
  ymax = 8.6,
  x post scale=2]
    \addplot[smooth, \twocolour] {betapdf(x,2,18)} node[above right] at (axis cs:0.0556,7.19){$a = 1$};
    \addplot[smooth, \twocolour!50] {betapdf(x,6,18)} node[above] at (axis cs:0.23,4.58){$a = 2$};
    \addplot[smooth, \twocolour!80] {betapdf(x,18,18)} node[above] at (axis cs:0.5,4.75){$a = 3$};
    \addplot[smooth, \twocolour] {betapdf(x,18,6)} node[above] at (axis cs:0.77,4.58){$a = 4$};
  \end{axis}
\end{tikzpicture}